\documentclass{article}
\usepackage{microtype}
\usepackage{mathtools}
\usepackage{mathrsfs}
\usepackage{algorithm}
\usepackage{algorithmic}
\usepackage{amssymb}
\usepackage{graphicx}
\usepackage[numbers,sort&compress]{natbib}
\usepackage[colorlinks=true, allcolors=black]{hyperref}
\usepackage{subcaption}
\usepackage{bm}
\usepackage{amsmath}
\numberwithin{equation}{section}
\usepackage{booktabs}
\usepackage{tabularx}
\usepackage{multirow}
\usepackage{enumerate}
\usepackage{amsthm}
\usepackage{authblk}

\usepackage{fancyhdr}
\newtheorem{theorem}{Theorem}[section]
\newtheorem{lemma}[theorem]{Lemma}
\newtheorem{proposition}[theorem]{Proposition}
\newtheorem{corollary}[theorem]{Corollary}

\theoremstyle{definition}
\newtheorem{definition}[theorem]{Definition}

\newtheorem{remark}[theorem]{Remark}
\allowdisplaybreaks

\def\e{\mathbb{E}}

\def\\ell{\mathcal{L}}
\def\rbb{\mathbb{R}}

\newcommand{\by}{\mathbf{y}}
\newcommand{\bz}{\mathbf{z}}

\newcommand{\bw}{\mathbf{w}}
\newcommand{\bfm}{\mathbf{m}}

\newcommand{\ibb}{\mathbb{I}}
\newcommand{\xcal}{\mathcal{X}}
\newcommand{\wcal}{\mathcal{W}}

\newcommand{\bx}{\mathbf{x}}

\newcommand{\bu}{\mathbf{u}}

\newcommand{\zcal}{\mathcal{Z}}

\newcommand{\ycal}{\mathcal{Y}}

\newcommand{\ebb}{\mathbb{E}}
\newcommand{\bg}{\mathbf{g}}

\newcommand{\bv}{\mathbf{v}}

\newcommand{\pbb}{\mathbb{P}}

\newcommand{\wzm}[1]{#1}

\title{Stochastic Gradient Descent with Momentum is Algorithmically Stable}
\author[1]{Yunwen Lei}
\author[2]{Zimeng Wang}
\author[1]{Xiaoming Yuan}
\affil[1]{Department of Mathematics, The University of Hong Kong, Hong Kong, China\\
\texttt{leiyw@hku.hk, xmyuan@hku.hk}}
\affil[2]{Department of Mathematics and Mathematical Statistics, Ume\r{a} University, Sweden\\
\texttt{zimeng.wang@umu.se}}

\date{}

\begin{document}

\maketitle
\begin{abstract}%

    Stochastic gradient descent with momentum (SGDM) is one of the most widely used optimization algorithms in machine learning. While optimization properties of SGDM have been extensively studied in the literature, it remains insufficiently understood whether and when SGDM can generalize well to unseen data. In particular, it has been conjectured that while momentum accelerates training, it may degrade generalization. In this paper, we close this gap by developing a comprehensive generalization analysis of SGDM through the lens of algorithmic stability. More specifically, we introduce a generalized SGDM framework that encompasses both Polyak’s and Nesterov’s momentum schemes, and establish tight on-average model stability bounds for smooth and convex problems. Notably, the obtained bounds exploit small optimization error bounds along the trajectory, apply to any momentum parameter in the interval $[0, 1)$, and do not require the commonly assumed Lipschitzness of loss functions. We further derive optimization error bounds for the generalized SGDM, and combine them with our generalization analyses to obtain optimal excess population risk bounds for SGDM with both Polyak’s and Nesterov’s momentum.
\end{abstract}

\noindent\textbf{Keywords:} Algorithmic stability, generalization analysis, momentum, stochastic gradient descent, excess population risk.

\section{Introduction}
The rapid growth of model complexity and data scale in modern machine learning has made stochastic gradient descent (SGD) \citep{robbins1951stochastic} the dominant optimization method for training large-scale models. Unlike traditional gradient descent, which requires computing full gradients over the entire dataset at each iteration, SGD leverages random sampling to drastically reduce per-iteration computational costs, enabling scalability to massive datasets and high-dimensional parameter spaces \citep{bottou2018optimization,lecun2015deep,nguyen2019new,zhang2004solving}. To further accelerate convergence and improve practical performance, momentum techniques such as Polyak's momentum \citep{polyak1964some} (a.k.a., heavy-ball momentum) and Nesterov's momentum \citep{nesterov1983method} have been introduced, leading to the widely adopted stochastic gradient descent with momentum (SGDM) method. By incorporating a momentum term that accumulates past gradients, SGDM often exhibits faster convergence than vanilla SGD. The momentum mechanism also helps navigate flat regions of the loss landscape and can potentially avoid local minima \citep{ochs2015ipiasco}. Consequently, SGDM and its variants have become a cornerstone of deep learning optimization~\citep{kingma2015adam, qian1999momentum, sutskever2013importance, yan2018unified} and are ubiquitously deployed in training state-of-the-art deep neural networks \citep{he2016deep, simonyan2014very}.

SGDM has been extensively studied from an optimization perspective, with convergence guarantees in both convex and nonconvex settings (see, e.g., \citet{loizou2020momentum, sebbouh2021almost, wang2023generalized, yang2016unified}). However, in machine learning, the goal extends beyond minimizing training loss; one also needs to ensure that the learned model generalizes well to unseen data~\citep{bousquet2002stability,hardt2016train,neu2021information}. Despite its practical importance, most existing analyses focus primarily on optimization error (i.e., the gap between the training loss of the algorithm's output and that of the optimal model), while leaving the generalization behavior (the difference between test performance and training performance) largely unexplored. Recent work has established sharp generalization bounds for SGD and its variants \citep{hardt2016train,lei2020fine} by leveraging the concept of algorithmic stability \citep{bousquet2002stability}, which quantifies how sensitive an algorithm's output is to perturbations in the training data. It was conjectured in \citet{hardt2016train} that while momentum may accelerate training, it might simultaneously degrade generalization. Nevertheless, the stability and generalization properties of SGDM remain relatively underexplored, with only a handful of recent studies \citep{attia2021algorithmic,chen2018stability,dang2025algorithmic,ramezani2024generalization} starting to investigate this important problem. These existing analyses are limited: they often address restricted scenarios (e.g., deterministic settings or quadratic objectives), provide only lower bounds demonstrating potential instability, or rely on additional structural assumptions such as Lipschitzness, strong convexity, or specific noise models.
As a result, a comprehensive understanding of whether and when SGDM generalizes well is still lacking, particularly for the general class of smooth and convex problems commonly encountered in machine learning.

In this paper, we aim to bridge this gap by providing a comprehensive stability and generalization analysis of a generalized SGDM framework that encompasses both SGD with Polyak's momentum and SGD with Nesterov's momentum. Our main contributions are summarized below. Although we focus on convex loss functions, our results can be readily extended to strongly convex settings by combining our arguments with those in \citet{hardt2016train, lei2020fine}.
\begin{itemize}
    \item[1).] We introduce a generalized SGDM framework and establish comprehensive on-average model stability \citep{lei2020fine} bounds for smooth convex objectives. Our bounds reveal a fundamental trade-off between momentum and stability: \emph{the momentum parameter $\beta$ maintains stability but may worsen it by a constant factor of order $O(1/(1-\beta)^{3/2})$}. Notably, our analysis does not require the commonly assumed Lipschitzness of loss functions; instead, it exploits the self-bounding property of smooth functions to control the gradient norms.
    \item[2).] Building on the general stability analysis of SGDM, we derive explicit stability bounds for both SGD with Polyak's momentum and SGD with Nesterov's momentum. The resulting bounds are tight in the sense that they match the known bounds for SGD \citep{hardt2016train, lei2020fine} up to a factor of $1/(1-\beta)^{3/2}$. To the best of our knowledge, these are the first tight stability guarantees for both momentum variants under comparable assumptions.
    \item[3).] We establish optimization error bounds for the generalized SGDM framework and combine them with the stability results to obtain excess population risk (EPR) bounds that characterize the overall learning performance. With appropriately chosen step sizes, we further obtain optimal EPR rates of order $O(1/\sqrt{n})$ for both momentum variants ($n$ is the sample size), which match the minimax lower bounds for statistical guarantees \citep{agarwal2009information}.
    \item[4).] Our analysis develops several techniques that may be of independent interest. To handle the complexities introduced by momentum terms in the stability analysis, we carefully bound these terms using weighted summations of gradient-based quantities. In addition, we construct an auxiliary sequence for the generalized SGDM framework and apply a self-bounding property to manage gradient-norm terms, substantially simplifying the optimization analysis. These tools may facilitate sharper analyses of other momentum-based algorithms.
\end{itemize}
The remainder of the paper is organized as follows. Section~\ref{sec:rw} reviews the related literature. In Section~\ref{sec:setup}, we introduce the problem setup and the generalized SGDM framework. Our main stability and optimization error bounds are presented in Section~\ref{sec:gen-opt}, and the corresponding EPR bounds are established in Section~\ref{sec:epr}. Technical proofs are deferred to Section~\ref{sec:proofs}. Section~\ref{sec:experiments} reports numerical experiments, and Section~\ref{sec:conclusion} concludes the paper.

\section{Related Work}\label{sec:rw}
In this section, we review related work on both the optimization and generalization analyses of the stochastic gradient descent with momentum (SGDM) method.
\subsection{Optimization Analysis}
Momentum methods have a rich history in optimization, dating back to the pioneering work of \citet{polyak1964some}, who introduced the heavy-ball method that accumulates past gradient information to accelerate deterministic gradient descent. Subsequently, \citet{nesterov1983method} proposed a different momentum scheme by evaluating the gradient at a look-ahead point. The resulting accelerated gradient method achieves the optimal convergence rate of $O(1/t^2)$ for smooth convex functions, where $t$ is the number of iterations. These deterministic momentum methods have been extensively studied and extended to various settings (see, e.g., \citet{beck2009fast, nesterov2013gradient, ochs2014ipiano}).

The success of momentum methods in deterministic optimization naturally motivated their extension to stochastic settings. The SGDM method seeks to combine the computational efficiency of SGD with the acceleration benefits of momentum. For SGD with Polyak's momentum, \citet{loizou2020momentum} established global linear convergence rates for convex quadratic objectives. When the initial point lies sufficiently close to an optimal solution, \citet{gitman2019understanding} showed that the method converges linearly for strongly convex and smooth functions. For convex and smooth objectives, \citet{sebbouh2021almost} proved almost sure convergence on the last iterate of SGD with Polyak's momentum under time-varying step sizes and momentum parameters.
A unified convergence analysis of SGDM that includes both SGD with Polyak's momentum and SGD with Nesterov's momentum for nonconvex smooth loss functions was given in \citet{yan2018unified, yang2016unified} under a bounded gradient assumption, and \citet{yu2019linear} extended these results to distributed settings. For stochastic composite optimization problems involving both smooth and nonsmooth components, \citet{lan2012optimal} introduced an accelerated stochastic gradient method based on Nesterov's momentum and proved its optimality for convex objectives; this approach was later extended to nonconvex problems in \citet{ghadimi2016accelerated}.

\subsection{Stability and Generalization Analysis}
Algorithmic stability is a fundamental concept in statistical learning theory for understanding the generalization behavior of learning algorithms. One of the most popular stability concepts is the uniform stability~\citep{bousquet2002stability,elisseeff2005stability,villa2013learnability}, which can imply high-probability generalization bounds~\citep{bousquet2020sharper,feldman2019high,zhang2023mathematical}. To derive generalization bounds in expectation, some weaker stability concepts have been introduced, including the on-average stability~\citep{shalev2010learnability} and the on-average model stability~\citep{lei2020fine}. The seminal work~\citep{hardt2016train} pioneered the stability analysis of SGD, which motivated a surge of interest in generalization analyses of stochastic optimization algorithms via algorithmic stability~\citep{charles2018stability,chen2024three,kuzborskij2018data,lei2020fine,nikolakakis2022beyond,schliserman2022stability}. While stability analysis has been primarily used to study convex and smooth learning problems, recent progress shows that it can also imply optimal rates for nonconvex models such as neural networks~\citep{deora2024optimization,richards2021stability,taheri2024generalization} and nonsmooth problems~\citep{bassily2020stability,lei2020fine}. %

We now review representative work on stability analysis specifically for SGDM.
The work~\citep{chen2018stability} considered convex, Lipschitz, and quadratic objective functions, and derived uniform stability bounds for the Heavy Ball (HB) method and Nesterov Accelerated Gradient (NAG) method in the deterministic setting. Later, the work~\citep{attia2021algorithmic} showed that NAG is unstable for general convex functions by demonstrating that the uniform stability parameter grows at least exponentially with the number of iterations.
The work~\citep{ramezani2024generalization} provided an upper bound on the uniform stability of SGD with Polyak's momentum for strongly convex problems, and a lower bound for convex problems. They also introduced a variant called SGD with early momentum to promote stability, which only includes the momentum term in the early stage.
The recent work~\citep{dang2025algorithmic} studied the stability of SGDM from the perspective of stochastic differential equations (SDEs). Their analysis requires injecting noise into the learning process, and measures the stability by the Wasserstein distance between distributions of these SDE solutions.

\section{Problem Setup}\label{sec:setup}
In this section, we introduce the optimization and generalization framework, define the algorithmic stability concepts used in our analysis, and present the generalized SGDM algorithm.
\subsection{Optimization and Generalization}
Let $\pbb$ be a probability measure defined on a sample space $\zcal := \xcal \times \ycal$, where $\xcal$ is an input space and $\ycal \subseteq \rbb$ is an output space. Let $S = \{\bz_1, \ldots, \bz_n\}$ be $n$ training examples drawn independently from $\pbb$, based on which we aim to construct a prediction model $h : \xcal \to \ycal$. We consider parameterized models indexed by a parameter $\bw \in \wcal$, where $\wcal := \rbb^d$ is the model space. We quantify the performance of a model $\bw$ on an example $\bz$ by a nonnegative loss function $\ell(\bw;\bz)$, where $\ell:\wcal\times\zcal\to\rbb_+$. Then, the behavior of the model $\bw$ on training examples and testing examples is quantified by the empirical risk $L_S(\bw)$ and the population risk $L(\bw)$ respectively, which are given by
\[
  L_S(\bw)=\frac{1}{n}\sum_{i=1}^{n}\ell(\bw;\bz_i)\quad\text{and}\quad L(\bw)=\ebb_{\bz}[\ell(\bw;\bz)].
\]
Here $\ebb_{\bz}[\cdot]$ denotes the expectation taken with respect to (w.r.t.) $\bz$.
To build a model, we typically apply a (stochastic) optimization algorithm $A$ to approximately minimize the empirical risk, and denote by $A(S)$ the resulting model. The relative performance of $A(S)$ compared with the optimal population-risk minimizer $\bw^* := \arg\min_{\bw \in \wcal} L(\bw)$ is measured by the \emph{excess population risk} $L(A(S))-L(\bw^*)$.

An effective way to control the excess risk is to decompose it into two components:
\begin{equation}\label{err-decomposition}
\ebb\big[L(A(S))-L(\bw^*)\big]=\ebb\big[L(A(S))-L_S(A(S))\big]+\ebb[L_S(A(S))-L_S(\bw^*)],
\end{equation}
where $\ebb[\cdot]$ denotes the expectation w.r.t. both $A$ and $S$. To derive Eq.~\eqref{err-decomposition}, we used the fact that $\ebb[L_S(\bw^*)] = L(\bw^*)$ because $\bw^*$ is independent of $S$. We refer to $L(A(S)) - L_S(A(S))$ as the \emph{generalization gap}, which quantifies the discrepancy between the test and training performance of the output model. The term $L_S(A(S)) - L_S(\bw^*)$ is the \emph{optimization error}, measuring how far the algorithm’s output is from the empirical risk minimizer. The generalization gap is central in statistical learning theory, while the optimization error is a fundamental quantity in optimization. Eq.~\eqref{err-decomposition} highlights that achieving small test error requires simultaneously controlling both the generalization gap and the optimization error.

\subsection{Algorithmic Stability}
We now introduce the stability concepts that will be used to analyze the generalization behavior of SGDM. Intuitively, an algorithm $A$ is considered stable if the output model $A(S)$ is insensitive to small perturbations in the training dataset. Several notions of stability have been proposed in the literature to derive different types of generalization guarantees. In this paper, we focus primarily on two of them: the uniform stability~\citep{bousquet2002stability} and the on-average model stability~\citep{lei2020fine}.
We write $S\sim \widetilde{S}$ to indicate that $S$ and $\widetilde{S}$ are neighboring datasets, meaning they differ in at most one example. Throughout, $\|\cdot\|$ denotes the Euclidean norm.
\begin{definition}[Uniform stability]
Let $\epsilon\geq0$. An algorithm $A$ is said to be $\epsilon$-uniformly stable if
\begin{equation}\label{unif-stab}
\sup_{S\sim \widetilde{S}}\sup_{\bz\in\zcal}\big|\ell(A(S);\bz)-\ell(A(\widetilde{S});\bz)\big|\leq\epsilon.
\end{equation}
\end{definition}
\begin{definition}[On-average model stability\label{def:model-stable}]
Let $S=\{\bz_1,\ldots,\bz_n\}$ and $S'=\{\bz_1',\ldots,\bz_n'\}$ be two datasets drawn independently from $\pbb$. For each $i\in[n]$, we construct $S^{(i)}$ by replacing the $i$-th element in S with $\bz_i'$, i.e.,
\[
  S^{(i)}:=\big\{\bz_1,\ldots,\bz_{i-1},\bz_i',\bz_{i+1},\ldots,\bz_n\big\}.
\]
An algorithm $A$ is said to be on-average $\epsilon$-model stable for $\epsilon\ge 0$ if
\[
\frac{1}{n}\sum_{i=1}^{n}\ebb\big[\|A(S)-A(S^{(i)})\|^2\big]\leq \epsilon^2.
\]
\end{definition}
According to these definitions, uniform stability is a relatively strong notion, since Eq.~\eqref{unif-stab} requires the bound to hold for all neighboring datasets and $\bz\in\zcal$. In contrast, on-average model stability is weaker, as it takes expectations over $S$ and $S'$ and captures the average effect of perturbing a single training example in $S$. Furthermore, uniform stability quantifies sensitivity in terms of the loss functions, whereas on-average model stability measures sensitivity through the distance between the output models.

Before presenting the connection between stability and generalization, we first recall the definitions of Lipschitz continuity, smoothness, and convexity.
\begin{definition}[Lipschitzness, smoothness and convexity]
Let $G,\alpha\geq0$ and $g:\wcal\mapsto\rbb$ be differentiable.
\begin{itemize}
  \item We say $g$ is $G$-Lipschitz continuous if $|g(\bw)-g(\bw')|\leq G\|\bw-\bw'\|,\forall \bw,\bw'\in\wcal$.
  \item We say $g$ is $\alpha$-smooth if $\|\nabla g(\bw)-\nabla g(\bw')\|\leq \alpha\|\bw-\bw'\|,\forall \bw,\bw'\in\wcal$.
  \item We say $g$ is convex if
  \[
  g(\bw)\geq g(\bw')+\langle\nabla g(\bw'),\bw-\bw'\rangle,\quad \forall\bw,\bw'\in\wcal.
  \]
\end{itemize}
\end{definition}
The following lemma establishes a connection between on-average model stability and generalization for smooth problems, which will be instrumental in our analysis.
\begin{lemma}[\citealt{lei2020fine}]\label{lem:general-stablity}
    Let $S, S^{\prime}$, and $S^{(i)}$ be constructed as in Definition \ref{def:model-stable}. If for any $\bz\in\zcal$, the function $\mathbf{w} \mapsto \ell(\mathbf{w} ; \bz)$ is nonnegative and $\alpha$-smooth, then for any $\rho>0$ it holds that
        \begin{equation*}
            \ebb_{S, A}[L(A(S))-L_S(A(S))] \leq \frac{\alpha}{\rho} \ebb_{S, A}[L_S(A(S))]+\frac{\alpha+\rho}{2 n} \sum_{i=1}^n \ebb_{S, S^{\prime}, A}[\|A(S^{(i)})-A(S)\|^2].
        \end{equation*}
\end{lemma}

To establish tight stability bounds without the Lipschitzness assumption, we rely heavily on the following two useful lemmas.
The first one gives the self-bounding property of smooth and nonnegative functions, which controls the gradient of functions by their function values. This self-bounding property has been widely used to derive optimistic rates for learning algorithms~\citep{lei2020fine,zhao2024adaptivity}.
\begin{lemma}[Self-bounding property~\citep{srebro2010smoothness}\label{lem:self-bounding}]
  Let $\alpha>0$. If $g:\wcal\mapsto\rbb$ is $\alpha$-smooth and nonnegative, then $\|\nabla g(\bw)\|^2\leq 2\alpha g(\bw)$ for any $\bw\in\wcal$.
\end{lemma}
The next lemma establishes the co-coercivity for convex and smooth functions.
\begin{lemma}[Co-coercivity\label{lem:coercivity}~\citep{nesterov2013introductory}]
Let $g:\wcal\to\rbb$ be $\alpha$-smooth and convex. Then
\begin{gather*}
\langle\bw-\bw',\nabla g(\bw)-\nabla g(\bw')\rangle\geq \frac{1}{\alpha}\|\nabla g(\bw)-\nabla g(\bw')\|^2,~\forall\bw,\bw'\in\wcal.\label{coercivity-2}
\end{gather*}
\end{lemma}

\subsection{Stochastic Gradient Descent with Momentum}
We consider a generalized version of stochastic gradient descent with momentum (SGDM), which extends standard SGD by incorporating a momentum term to accelerate optimization. Let $\bw_1\in\wcal$ be an initial point and set $\bfm_0=\mathbf{0}$. Denote by $\bg_t:=\nabla \ell(\bw_t;\bz_{i_t})$ the stochastic gradient at iteration $t$, where $i_t$ is drawn uniformly from $[n]:=\{1,\dots,n\}$. The generalized SGDM method updates the iterates via
\begin{align}
&\bfm_t=\beta \bfm_{t-1}+\bg_t, \label{eq:m-update}\\
&\bw_{t+1}=\bw_t-\gamma \bg_t-\eta \bfm_t,\label{eq:w-update}
\end{align}
where $\gamma\ge0,~\eta>0$ denote step sizes and $\beta\in[0,1)$ is a momentum parameter.
This formulation is highly general and encompasses several important special cases. When $\beta=0$, the generalized SGDM method reduces to standard SGD with constant step size $\gamma+\eta>0$. Moreover, it includes both SGD with Polyak's momentum and SGD with Nesterov's momentum as special cases. In particular, if $\gamma=0$, the method reduces to SGD with Polyak’s (heavy-ball) momentum:
\begin{equation}\label{alg:sgd-polyak}
\tag{SGD-HB}
    \begin{aligned}
        &\bfm_t=\beta \bfm_{t-1}+\bg_t,\\
        &\bw_{t+1}=\bw_t-\eta \bfm_t.
    \end{aligned}
\end{equation}
Additionally, by setting $\eta = \beta \gamma$, the generalized SGDM method Eq.~\eqref{eq:m-update}–\eqref{eq:w-update} can be rewritten as
\begin{align*}
& \bfm_t=\beta \bfm_{t-1}+\bg_t, \\
& \bv_t=\beta \bfm_t+\bg_t, \\
& \bw_{t+1}=\bw_t-\gamma \bv_t.
\end{align*}
As shown in \citet[Section 2.3.1]{yu2019linear}, the above iterative scheme is equivalent to
\begin{equation}\label{alg:sgd-nesterov}
\tag{SGD-N}
    \begin{aligned}
        &\bu_t =\bw_t-\gamma \bg_t, \\
        &\bw_{t+1} =\bu_t+\beta(\bu_t-\bu_{t-1}),
    \end{aligned}
\end{equation}
which corresponds to SGD with Nesterov's momentum. Consequently, our analysis applies directly to standard SGD, SGD with Polyak’s momentum, and SGD with Nesterov’s momentum.
\wzm{
\begin{remark}\normalfont\label{rmk:tv-momentum}
Our main stability and excess population risk analysis focuses on constant momentum. Note that while the time-varying schedule $\beta_t=t/(t+3)=1-\Theta(1/t)$ in \eqref{alg:sgd-nesterov} yields the classical optimal $O(t^{-2})$ rate for deterministic smooth convex optimization~\citep{nesterov1983method}, it requires exact gradients and hence the resulting deterministic guarantee does not carry over to our stochastic settings. In particular, the error accumulation due to inexact gradients can slow down convergence or even cause divergence~\citep{wang2022scheduled}. 
Nevertheless, we provide a polynomial-order stability bound for \eqref{alg:sgd-nesterov} in Proposition~\ref{prop:stab-cvx-nesterov-tv-convex}, which improves upon the exponential-order bound in \citet{attia2021algorithmic} for smooth convex functions.
\end{remark}}

\section{Generalization and Optimization Analysis of the Generalized SGDM}\label{sec:gen-opt}
In this section, we present a comprehensive analysis of the generalized SGDM algorithm for convex and smooth problems, where we always assume the loss function is nonnegative. We establish both stability and optimization error bounds for the general SGDM framework. As special cases, we obtain the corresponding results for both SGD with Polyak's momentum and SGD with Nesterov's momentum. The omitted proofs of the results established in Sections \ref{sec:stab} and \ref{sec:opt} can be found in Sections \ref{sec:proof-stab} and \ref{sec:proof-opt}, respectively.

\subsection{Stability Bounds}\label{sec:stab}
We begin by analyzing the on-average model stability of the generalized SGDM, which is essential for deriving generalization bounds. Due to symmetry, we fix an index $i\in[n]$ and estimate $\ebb[\|\bw_t-\bw_t^{(i)}\|^2]$, where $\bw_t^{(i)}$ is derived by the generalized SGDM applied to $S^{(i)}$. Let $\bg_t^{(i)}$ and $\bfm_t^{(i)}$ be defined in the same way as $\bg_t$ and $\bfm_t$ except that they are derived based on $S^{(i)}$ instead of $S$. For notational convenience, we denote
\begin{equation}\label{eq:Gt-def}
G_t:= \bw_t-\gamma \bg_t,\quad G_t^{(i)}:=\bw_t^{(i)}-\gamma \bg_t^{(i)},
\end{equation}
which can be interpreted as SGD-type updates. Initializing with $\bw_1=\bw_1^{(i)}\in\wcal$ and $\bfm_0=\bfm_0^{(i)}=\mathbf{0}$, then it follows from Eq.~\eqref{eq:w-update} that
\begin{equation}\label{eq:cvx-stab-n1}
\bw_{t+1}=G_t-\eta \bfm_t,\quad \bw_{t+1}^{(i)}=G_t^{(i)}-\eta \bfm_t^{(i)}.
\end{equation}
This reformulation serves as a starting point for our stability analysis.

Before establishing the stability bounds of the generalized SGDM, we first introduce the following lemma. It decomposes the terms involving momentum into expressions based on gradients. %
\begin{lemma}\label{lem:ip-w-m}
Let $S, S^{(i)}$ be defined in Definition~\ref{def:model-stable}. Let $\{\bw_t\}_{t\ge 1}$ and $\{\bw_t^{(i)}\}_{t\ge 1}$ be the sequences produced by the generalized SGDM Eq.~\eqref{eq:m-update}-\eqref{eq:w-update} with $\beta\in [0,1)$ applied to $S$ and $S^{(i)}$, respectively. Then for any $t\ge 1$ and $\delta>0$, we have
\begin{align}
    \|\bfm_t-\bfm_t^{(i)}\|^2 \leq &(1+1/\delta) \sum_{k=1}^t \big((1+\delta)\beta^2\big)^{t-k}\|\bg_k-\bg_k^{(i)}\|^2,\label{eq:m-square}\\
    \langle \bw_t-\bw_t^{(i)}, \bfm_t-\bfm_t^{(i)}\rangle \ge & -(1+1 / \delta)\Big(\frac{\gamma}{2}+\eta\Big) \sum_{k=1}^{t-1}\Big(\sum_{l=k}^{t-1} \beta^{t-l}((1+\delta) \beta^2)^{l-k}\Big)\|\bg_k-\bg_k^{(i)}\|^2\notag\\
    &-\frac{\gamma}{2} \sum_{k=1}^{t-1} \beta^{t-k}\|\bg_k-\bg_k^{(i)}\|^2 + \sum_{k=1}^t \beta^{t-k}\langle \bw_k-\bw_k^{(i)}, \bg_k-\bg_k^{(i)}\rangle\label{eq:cvx-stab-5}.
\end{align}
\end{lemma}

Now we are ready to present the on-average model stability bounds for the generalized SGDM in the upcoming theorem. The obtained upper bound is data-dependent in the sense that it involves an accumulated summation of training errors during the optimization process, which can be further bounded by optimization error analysis. Importantly, this demonstrates the benefit of optimization in improving stability.
\begin{theorem}[Stability bounds for generalized SGDM]\label{thm:stab-cvx}
Let $S, S^{(i)}$ be defined in Definition~\ref{def:model-stable}. Let $\{\bw_t\}_{t\ge 1}$ and $\{\bw_t^{(i)}\}_{t\ge 1}$ be the sequences produced by the generalized SGDM Eq.~\eqref{eq:m-update}-\eqref{eq:w-update} with $\beta\in [0,1)$ applied to $S$ and $S^{(i)}$, respectively.
Assume that for any $\bz\in\zcal$, the map $\bw\mapsto\ell(\bw;\bz)$ is nonnegative, convex, and $\alpha$-smooth. If the step sizes $\gamma\ge 0$ and $\eta>0$ are chosen such that
\begin{equation}\label{eq:stab-lr}
    \frac{(1+\beta)(3-\beta)}{(1-\beta)^2}\eta+\frac{\beta^2+3}{2(1-\beta)^2}\gamma\leq \frac{1}{\alpha},
\end{equation}
then for any $t\ge 1$, we have
\begin{align}
&\ebb[\|\bw_{t+1}-\bw_{t+1}^{(i)}\|^2] \notag \\
\leq &\Big(\frac{2(1+\beta)^2}{(1-\beta)^3}\eta^2+2\gamma^2+\frac{\beta(\beta^2+3)}{(1-\beta)^3}\gamma\eta+\frac{(\gamma+\eta)((1-\beta)\gamma+\eta)t}{(1-\beta)^2 n}\Big)\frac{8 \alpha e}{n} \sum_{k=1}^t \ebb[L_S(\bw_k)],\label{eq:cvx-stab-sgdm}
\end{align}
where $e$ denotes Euler's number.
\end{theorem}

\wzm{
\begin{remark}\normalfont\label{rmk:empirical-risk}
    The empirical-risk term $L_S(\bw_k)$ in Eq.~\eqref{eq:cvx-stab-sgdm} is standard in stability analyses for smooth loss functions~\citep{srebro2010smoothness,kuzborskij2018data,lei2020fine}. Importantly, a smoothness assumption is in general not sufficient to replace the empirical risk with the excess empirical risk, as this replacement can imply a fast rate of order $O(1/n)$, which violates the minimax lower bound $O(1/\sqrt{n})$~\citep[Theorem 7]{chen2018stability}. Additional assumptions such as the interpolation \citep[Assumption 1]{ma2018power} and strong growth condition~\citep{vaswani2019fast} are required to get fast rates. Indeed, under the $\rho$-strong growth condition, one can apply the self-bounding property of smooth $L_S$ to obtain $\frac{1}{n}\sum_{i=1}^n\|\nabla \ell(\bw;\mathbf{z}_i)\|^2\leq \rho\|\nabla L_S(\bw)\|^2\leq 2L\rho(L_S(\bw)-\min_{\bw}L_S(\bw))$.
\end{remark}
}

Having established the general stability bounds for SGDM, we now apply this result to two specific algorithms: SGD with Polyak's momentum \eqref{alg:sgd-polyak} and SGD with Nesterov's momentum \eqref{alg:sgd-nesterov}.
Recall that $e$ denotes Euler's number.
\begin{corollary}[Stability bounds for \eqref{alg:sgd-polyak}]\label{thm:stab-cvx-polyak}
Let $S, S^{(i)}$ be defined in Definition~\ref{def:model-stable}. Let $\{\bw_t\}_{t\ge 1}$ and $\{\bw_t^{(i)}\}_{t\ge 1}$ be the sequences produced by SGD with Polyak's momentum \eqref{alg:sgd-polyak} with $\beta\in[0,1)$ applied to $S$ and $S^{(i)}$, respectively.
Assume that for any $\bz\in\zcal$, the map $\bw\mapsto\ell(\bw;\bz)$ is nonnegative, convex, and $\alpha$-smooth. If the step size $\eta>0$ is chosen such that
\[
\eta \leq \frac{(1-\beta)^2}{\alpha(1+\beta)(3-\beta)},
\]
then for any $t\ge 1$, we have that
\begin{equation}\label{eq:cvx-stab-polyak}
    \ebb[\|\bw_{t+1}-\bw_{t+1}^{(i)}\|^2]\leq \Big(\frac{2(1+\beta)^2}{(1-\beta)^3}+\frac{t}{(1-\beta)^2 n}\Big)\frac{8 \alpha e\eta^2}{n} \sum_{k=1}^t \ebb[L_S(\bw_k)].
\end{equation}
\end{corollary}
\begin{proof}
    Setting $\gamma=0$ in Eq.~\eqref{eq:cvx-stab-sgdm} in Theorem \ref{thm:stab-cvx} directly gives Eq.~\eqref{eq:cvx-stab-polyak}. Additionally, plugging $\gamma=0$ into Eq.~\eqref{eq:stab-lr} leads to $\eta \leq \frac{(1-\beta)^2}{\alpha(1+\beta)(3-\beta)}$, which coincides with the given step size condition, hence completes the proof.
\end{proof}

\begin{remark}\label{rm:cvx-sgd}\normalfont
According to Eq.~\eqref{eq:cvx-stab-polyak}, we know that the stability bound is an increasing function of $\beta$. Therefore, increasing $\beta$ worsens the stability of SGD with Polyak's momentum.
If $\beta=0$, then SGD with Polyak's momentum reduces to the vanilla SGD. In this case, Corollary~\ref{thm:stab-cvx-polyak} implies
\begin{equation}\label{sgd-stability-bound}
\ebb[\|\bw_{t+1}-\bw_{t+1}^{(i)}\|^2]\leq \Big(2+\frac{t}{n}\Big)\frac{8 \alpha e\eta^2}{n} \sum_{k=1}^t \ebb[L_S(\bw_k)].
\end{equation}
This matches the existing stability analysis of SGD in \citet{lei2020fine}, which validates the effectiveness of our general analysis of SGDM, as it recovers the existing result as a special case. It is also clear that the bounds in Eq.~\eqref{eq:cvx-stab-polyak} and Eq.~\eqref{sgd-stability-bound} match up to a factor of $1/(1-\beta)^3$. That means, our analysis reveals that introducing the momentum worsens the stability by at most a factor of $1/(1-\beta)^{3/2}$.
\end{remark}

The next corollary specializes our general result to SGD with Nesterov's momentum by setting $\eta=\beta\gamma$ in Theorem \ref{thm:stab-cvx}.
\begin{corollary}[Stability bounds for \eqref{alg:sgd-nesterov}]\label{thm:stab-cvx-nesterov}
Let $S, S^{(i)}$ be defined in Definition~\ref{def:model-stable}. Let $\{\bw_t\}_{t\ge 1}$ and $\{\bw_t^{(i)}\}_{t\ge 1}$ be the sequences produced by SGD with Nesterov's momentum \eqref{alg:sgd-nesterov} with $\beta\in(0, 1)$ applied to $S$ and $S^{(i)}$, respectively.
Assume that for any $\bz\in\zcal$, the map $\bw\mapsto\ell(\bw;\bz)$ is nonnegative, convex, and $\alpha$-smooth. If the step size $\gamma>0$ is chosen such that
\[\gamma\leq\frac{2(1-\beta)^2}{\alpha(3+6\beta+5\beta^2-2\beta^3)},\]
then for any $t\ge 1$, we have
\begin{equation}\label{eq:cvx-stab-nesterov}
\ebb[\|\bw_{t+1}-\bw_{t+1}^{(i)}\|^2]\leq \Big(2+\frac{\beta^2(3\beta^2+4\beta+5)}{(1-\beta)^3}+\frac{(1+\beta) t}{(1-\beta)^2 n}\Big)\frac{8 \alpha e\gamma^2}{n} \sum_{k=1}^t \ebb[L_S(\bw_k)].
\end{equation}
\end{corollary}

\begin{proof}
    Setting $\eta=\beta\gamma$ in Eq.~\eqref{eq:cvx-stab-sgdm} in Theorem \ref{thm:stab-cvx} directly gives Eq.~\eqref{eq:cvx-stab-nesterov} (note $(1-\beta)\gamma+\eta=\gamma$). Then it remains to verify the step size condition Eq.~\eqref{eq:stab-lr}. Plugging $\eta=\beta\gamma$ into Eq.~\eqref{eq:stab-lr} gives
    \begin{align*}
        \frac{1}{\alpha} &\ge \frac{\beta(1+\beta)(3-\beta)}{(1-\beta)^2}\gamma+\frac{\beta^2+3}{2(1-\beta)^2}\gamma
        =\frac{3+6\beta+5\beta^2-2\beta^3}{2(1-\beta)^2}\gamma,
    \end{align*}
    which is clearly satisfied with the given condition on $\gamma$. The proof is thus complete.
\end{proof}

\subsubsection{Discussions with Existing Work}
In this subsection, we compare our results with existing work on the stability of momentum-based methods. The work \citep{chen2018stability} analyzed the stability of the Heavy Ball (HB) method and Nesterov's Accelerated Gradient (NAG) method, which are deterministic versions of SGD with Polyak's momentum and SGD with Nesterov's momentum, respectively (i.e., replacing $\bg_t$ with $\nabla L_S(\bw_t)$ in \eqref{alg:sgd-polyak} and \eqref{alg:sgd-nesterov}).
Under the assumptions of $G$-Lipschitzness, $\alpha$-smoothness, and convex quadratic losses, \citet{chen2018stability} established that NAG is $\epsilon_N$-uniformly stable and HB is $\epsilon_H$-uniformly stable, with
  \begin{equation}\label{chen2018}
    \epsilon_N\leq \frac{4\gamma G^2t^2}{n}\quad\text{and}\quad \epsilon_H\leq \frac{4\eta G^2t}{(1-\sqrt{\alpha})n}.
  \end{equation}
These bounds imply that NAG has a stability rate of order $O(\gamma t^2 / n)$, whereas HB enjoys a rate of order $O(\eta t / n)$.
Because the bound on $\epsilon_N$ grows quadratically in $t$ while the bound on $\epsilon_H$ grows linearly, NAG appears less stable than HB in this deterministic quadratic setting.
A key limitation of \citet{chen2018stability} is that its stability analysis applies only to quadratic loss functions. The authors conjectured that the bounds in Eq.~\eqref{chen2018} extend to general convex objectives, but this conjecture was later refuted by \citet{attia2021algorithmic}, which demonstrated that for NAG with momentum $\beta_t = (t-1)/(t+2)$, the uniform stability parameter can grow at least exponentially in $t$ on a carefully constructed convex, Lipschitz, and smooth problem.

Our stability analysis does not contradict the instability of NAG established in \citet{attia2021algorithmic}. Several key differences explain this discrepancy. First, \citet{attia2021algorithmic} analyzed deterministic gradient methods, whereas our work focuses on stochastic gradient methods. The behavior of NAG can differ substantially between deterministic and stochastic regimes. For instance, although NAG accelerates deterministic gradient descent from a rate of $O(1/t)$ to $O(1/t^2)$ for smooth convex objectives, such acceleration is known not to extend to stochastic algorithms. Second, the lower bound in \citet{attia2021algorithmic} concerns uniform stability, which is significantly stronger than the on-average model stability considered in our analysis. Uniform stability measures the worst-case change in the output across all pairs of neighboring datasets and test examples, whereas on-average model stability evaluates the expected stability behavior. \wzm{As we show in Proposition~\ref{prop:stab-cvx-nesterov-tv-convex}, for time-varying Nesterov momentum with arbitrary choices of $\{\beta_t\}_{t\ge1}$ between 0 and 1, the on-average model stability bound is at most polynomial order, which improves over the exponential lower bound on uniform stability in~\citet{attia2021algorithmic}.} 

We next compare our results with the stability analysis of SGD with Polyak's momentum \eqref{alg:sgd-polyak} established in \citet{ramezani2024generalization}. This prior work investigated the uniform stability of the method under the assumptions that the loss function is $\mu$-strongly convex, $G$-Lipschitz, and $\alpha$-smooth. Under the conditions
  $\frac{\eta\alpha\mu}{\alpha+\mu}-\frac{1}{2}\leq \beta<\frac{\eta\alpha\mu}{3(\alpha+\mu)}$ and $\eta\leq \frac{2}{\alpha+\mu}$, it was shown that \eqref{alg:sgd-polyak} is $\epsilon_u$-uniformly stable with
  \begin{equation*}\label{ramezani2024}
  \epsilon_u\leq \frac{2\eta G^2(\alpha+\mu)}{n(\eta\alpha\mu-3\beta(\alpha+\mu))}.
  \end{equation*}
  A main limitation of this result lies in the highly restrictive constraint on the momentum parameter: $\beta<\frac{\eta\alpha\mu}{3(\alpha+\mu)}$. In many practical machine learning settings, the strong convexity parameter $\mu$ is extremely small. Indeed, generalization theory (e.g., \citep{shalev2010learnability}) often motivates choosing $\mu \asymp 1/\sqrt{n}$ to achieve optimal excess risk rates. Moreover, the step size $\eta$ must typically be small to ensure convergence; for instance, \citet{ramezani2024generalization} suggests $\eta \asymp t^{-1/2}$. Under these realistic choices, the constraint above implies $\beta\lesssim \eta\mu\lesssim \frac{1}{(nt)^{1/2}}$, meaning that the stability analysis in \citet{ramezani2024generalization} applies only when the momentum parameter is extremely small, far below the values commonly used in practice. Consequently, their analysis cannot capture or explain the practical role of momentum. In contrast, our analysis applies to all $\beta\in[0, 1)$, thereby providing stability guarantees that genuinely capture the effect of momentum used in real-world applications. Beyond this, our results offer several additional advantages. We relax the strong convexity assumption in \citet{ramezani2024generalization} to mere convexity and eliminate the Lipschitzness requirement entirely. Moreover, our stability bounds involve summations of the training losses $L_S(\bw_k)$, which are typically small since we are optimizing the training errors. This explicitly highlights how optimization contributes to generalization. Finally, \citet{ramezani2024generalization} constructed a one-dimensional convex problem showing that the uniform stability of \eqref{alg:sgd-polyak} is at least $\frac{2\eta G^2}{n}\sum_{j=0}^{t-1}(t-j)\beta^j$. However, no corresponding upper bound was established for general convex problems. Our analysis fills this gap.

A different line of analysis was recently developed in \citet{dang2025algorithmic}, which studied the stability of SGDM through a stochastic differential equation (SDE) approach.
Their work considered SGDM with heavy-tailed gradient noise modeled by an $\alpha$-stable L\'evy process. Assuming dissipativity, Lipschitz continuity, and a pseudo-Lipschitz-like condition on the gradients, they derived stability bounds in terms of the Wasserstein distance between the distributions of the corresponding L\'evy-driven SDEs constructed from two neighboring datasets. In contrast, our analysis establishes stability bounds directly at the algorithmic level, without introducing additional noise models, and does not require the dissipativity condition.

Finally, unlike most existing studies on momentum-based methods~\citep{attia2021algorithmic,chen2018stability,ramezani2024generalization} which focus on uniform stability, we analyze the on-average model stability of SGDM. This weaker notion of stability enables us to remove the Lipschitz assumption and to incorporate the training errors along the optimization trajectory into the stability bounds.

\subsection{Optimization Analysis}\label{sec:opt}
We proceed to analyze the optimization error bounds of the generalized SGDM algorithm. To facilitate the analysis, we introduce an auxiliary sequence $\{\by_k\}_{k\ge 1}$ defined by
\begin{equation}\label{eq:y-def}
\by_k= \begin{cases}\bw_1, \quad \text{if } k=1 &  \\ \frac{1}{1-\beta} \bw_k-\frac{\beta}{1-\beta} \bw_{k-1}+\frac{\beta \gamma}{1-\beta} \bg_{k-1}, & \text{if } k \ge 2.\end{cases}
\end{equation}
The following lemma shows that this auxiliary sequence $\{\by_k\}_{k\ge 1}$ evolves according to a simple SGD update rule, which facilitates the optimization analysis significantly. The proofs of results in this section are given in Section~\ref{sec:proof-opt}.
\begin{lemma}\label{lem:y-gd}
Consider the sequence $\{\by_k\}_{k\ge 1}$ defined in Eq.~\eqref{eq:y-def}, where $\{\bw_k\}_{k\ge 1}$ is generated by the generalized SGDM Eq.~\eqref{eq:m-update}-\eqref{eq:w-update}. Then for any $k\ge 1$ and $\beta\in [0,1)$, it holds that
\[
\by_{k+1}=\by_k-\Big(\gamma+\frac{\eta}{1-\beta}\Big) \bg_k.
\]
\end{lemma}

The next lemma provides an upper bound on the distance between $\by_k$ and $\bw_k$, which will be used in the optimization analysis.
\begin{lemma}\label{lem:dist-yw}
    Consider the sequence $\{\by_k\}_{k\ge 1}$ defined in Eq.~\eqref{eq:y-def}, where $\{\bw_k\}_{k\ge 1}$ is generated by the generalized SGDM Eq.~\eqref{eq:m-update}-\eqref{eq:w-update}. Then for any $k\ge 1$ and $\beta\in[0,1)$, it holds that
    \[
    \|\by_k-\bw_k\|^2 \leq \frac{\beta^2}{(1-\beta)^3} \eta^2 \sum_{j=1}^{k-1} \beta^{k-1-j}\|\bg_j\|^2.
    \]
\end{lemma}

Leveraging the above two lemmas, we can establish the optimization error bounds for the generalized SGDM algorithm, as summarized in the following theorem. Note that $\ebb_A[\cdot]$ denotes expectation taken only with respect to the stochasticity of the algorithm, rather than the randomness of the training set $S$.

\begin{theorem}[Optimization error bounds for generalized SGDM]\label{thm:cvx-opt}
Assume for any $\bz\in\zcal$, the map $\bw\mapsto\ell(\bw;\bz)$ is nonnegative, convex, and $\alpha$-smooth.
Let $\{\bw_t\}_{t\ge 1}$ be the sequence produced by the generalized SGDM Eq.~\eqref{eq:m-update}-\eqref{eq:w-update} with $\beta\in [0, 1)$ applied to $S$ and denote  $\bar{\bw}_t:=\frac{1}{t}\sum_{k=1}^t \bw_k$. If the step sizes $\gamma\ge 0$ and $\eta>0$ are chosen such that
\begin{equation}\label{eq:opt-lr}
    \gamma^2+\frac{2}{1-\beta}\gamma\eta+\frac{1+\beta}{(1-\beta)^2}\eta^2\leq \frac{2(1-\beta)\gamma+2\eta}{3\alpha},
\end{equation}
then for any $\bw\in\wcal$ and $t\ge 1$, we have
\begin{align*}
   \ebb_A[L_S(\bar{\bw}_t)]-L_S(\bw)
   \leq\frac{3(1-\beta)\|\bw_1-\bw\|^2}{2\big((1-\beta)\gamma+\eta\big) t}+ \Big((1-\beta)\gamma+\eta +\frac{\beta\eta^2}{(1-\beta)\gamma+\eta}\Big)\frac{3\alpha L_S(\bw)}{(1-\beta)^2}.
\end{align*}
\end{theorem}

By applying Theorem \ref{thm:cvx-opt}, we obtain optimization error bounds for both SGD with Polyak's momentum (by setting $\gamma=0$) and SGD with Nesterov's momentum (by setting $\eta=\beta\gamma$). We summarize these results in the following corollaries and omit their proofs for brevity.
\begin{corollary}[Optimization error bounds for \eqref{alg:sgd-polyak}]
Assume for any $\bz\in\zcal$, the map $\bw\mapsto\ell(\bw;\bz)$ is nonnegative, convex, and $\alpha$-smooth.
Let $\{\bw_t\}_{t\ge 1}$ be the sequence produced by SGD with Polyak's momentum \eqref{alg:sgd-polyak} with $\beta\in [0, 1)$ applied to $S$ and denote $\bar{\bw}_t:=\frac{1}{t}\sum_{k=1}^t \bw_k$. If we choose the step size
\[
\eta\leq \frac{2(1-\beta)^2}{3\alpha(1+\beta)},
\]
then for any $\bw\in \wcal$ and $t\ge 1$, we have
\[
\ebb_A[L_S(\bar{\bw}_t)]-L_S(\bw)\leq\frac{3(1-\beta)\|\bw_1-\bw\|^2}{2\eta t}+\frac{3\alpha(1+\beta)}{(1-\beta)^2}\eta L_S(\bw).
\]
\end{corollary}

\begin{corollary}[Optimization error bounds for \eqref{alg:sgd-nesterov}]
Assume for any $\bz\in\zcal$, the map $\bw\mapsto\ell(\bw;\bz)$ is nonnegative, convex, and $\alpha$-smooth.
Let $\{\bw_t\}_{t\ge 1}$ be the sequence produced by SGD with Nesterov's momentum \eqref{alg:sgd-nesterov} with $\beta\in (0, 1)$ applied to $S$ and denote $\bar{\bw}_t:=\frac{1}{t}\sum_{k=1}^t \bw_k$. If we choose the step size
\[
\gamma \leq \frac{2(1-\beta)^2}{3\alpha(1+\beta^3)},
\]
then for any $\bw\in\wcal$ and $t\ge 1$, we have
\[
\ebb_A[L_S(\bar{\bw}_t)]-L_S(\bw)\leq\frac{3(1-\beta)\|\bw_1-\bw\|^2}{2\gamma t}+ \frac{3\alpha(\beta^3+1)}{(1-\beta)^2}\gamma L_S(\bw).
\]
\end{corollary}

\section{Excess Population Risk of the Generalized SGDM}\label{sec:epr}
In this section, we combine the stability and optimization results from the previous section to derive excess population risk (EPR) bounds for the generalized SGDM algorithm. By specifying appropriate step sizes, we further obtain EPR bounds for SGD with Polyak's momentum and SGD with Nesterov's momentum. Throughout, we write $A\lesssim B$ to indicate that there exists a universal constant $c>0$ such that $A\leq cB$, and we denote $A\asymp B$ when both $A\lesssim B$ and $B\lesssim A$ hold. The omitted proofs of the results presented in this section can be found in Section \ref{sec:proof-epr}.

By combining the results established in Theorem \ref{thm:stab-cvx} and Theorem \ref{thm:cvx-opt} and leveraging the error decomposition in Eq.~\eqref{err-decomposition}, we characterize the EPR bounds of the generalized SGDM in the following theorem.
\begin{theorem}[Excess population risk of generalized SGDM\label{thm:epr-cvx}]
    Assume that for any $\bz\in\zcal$, the map $\bw\mapsto\ell(\bw;\bz)$ is nonnegative, convex, and $\alpha$-smooth.
    Let $\{\bw_t\}_{t\ge 1}$ be the sequence produced by the generalized SGDM Eq.~\eqref{eq:m-update}-\eqref{eq:w-update} with $\beta \in[0,1)$ applied to $S$ and denote $\bar{\bw}_t:=\frac{1}{t}\sum_{k=1}^t \bw_k$. If the step sizes $\gamma\ge 0$ and $\eta>0$ are chosen such that
\begin{align}
    & \frac{(1+\beta)(3-\beta)}{(1-\beta)^2}\eta+\frac{\beta^2+3}{2(1-\beta)^2}\gamma\leq \frac{1}{\alpha},\label{eq:epr-lr1}\\
    & \gamma^2+\frac{2}{1-\beta}\gamma\eta+\frac{1+\beta}{(1-\beta)^2}\eta^2\leq \frac{2(1-\beta)\gamma+2\eta}{3\alpha},\label{eq:epr-lr2}
\end{align}
    then for any $t\ge 1$ and $\rho>0$, we have
\begin{align*}
& \e[L(\bar{\bw}_t)]-L(\bw^*) \notag\\
\lesssim & \frac{1-\beta}{\big((1-\beta)\gamma+\eta\big) t}+ D(\beta, \gamma, \eta) L(\bw^*) + \frac{1}{\rho}\Big(L(\bw^*)+\frac{1-\beta}{\big((1-\beta)\gamma+\eta\big) t}+ D(\beta, \gamma, \eta) L(\bw^*)\Big) \\
&+ \frac{(1+\rho) C(\beta, \gamma, \eta, t)}{n} \Big(t L(\bw^*)+\frac{1-\beta}{(1-\beta)\gamma+\eta}+ t D(\beta, \gamma, \eta) L(\bw^*)\Big),
\end{align*}
where we define
\begin{align*}
    & C(\beta, \gamma, \eta, t):=\frac{\eta^2}{(1-\beta)^3}+\gamma^2+\frac{\beta \gamma\eta}{(1-\beta)^3}+\frac{(\gamma+\eta)((1-\beta)\gamma+\eta)t}{(1-\beta)^2 n},\\
    & D(\beta, \gamma, \eta):=\frac{(1-\beta)\gamma+\eta}{(1-\beta)^2} +\frac{\beta\eta^2}{(1-\beta)^2\big((1-\beta)\gamma+\eta\big)}.
\end{align*}
\end{theorem}

Theorem~\ref{thm:epr-cvx} establishes a general result with various choices of the step sizes $\gamma\ge 0$ and $\eta> 0$. Next, we focus on SGD with Polyak's momentum \eqref{alg:sgd-polyak} and SGD with Nesterov's momentum \eqref{alg:sgd-nesterov} to obtain simplified results. We begin with the case of Polyak's momentum, stated in the next theorem. %
\begin{theorem}[Excess population risk of \eqref{alg:sgd-polyak}]\label{thm:epr-polyak}
    Assume that for any $\bz\in\zcal$, the map $\bw\mapsto\ell(\bw;\bz)$ is nonnegative, convex, and $\alpha$-smooth.
    Let $\{\bw_t\}_{t\ge 1}$ be the sequence produced by SGD with Polyak's momentum \eqref{alg:sgd-polyak} with $\beta\in [0,1)$ applied to $S$ and denote $\bar{\bw}_t:=\frac{1}{t}\sum_{k=1}^t \bw_k$. If the step size $\eta>0$ is chosen such that
    \[\eta\leq\frac{(1-\beta)^2}{\alpha(1+\beta)(3-\beta)},\]
    then for any $t\ge 1$ and $\rho>0$, we have
\begin{multline*}
\e[L(\bar{\bw}_t)]-L(\bw^*) \lesssim \frac{1-\beta}{\eta t}+ \frac{\eta}{(1-\beta)^2} L(\bw^*) + \frac{1}{\rho}\Big(L(\bw^*)+\frac{1-\beta}{\eta t}+ \frac{\eta}{(1-\beta)^2} L(\bw^*)\Big) \\
+ \frac{1+\rho}{n} \frac{\eta^2}{(1-\beta)^2}\Big(\frac{1}{1-\beta}+\frac{t}{n}\Big) \Big(t L(\bw^*)+\frac{1-\beta}{\eta}+ \frac{\eta t}{(1-\beta)^2}  L(\bw^*)\Big).
\end{multline*}
\end{theorem}

Theorem~\ref{thm:epr-polyak} establishes general EPR bounds for \eqref{alg:sgd-polyak} with various choices of step size $\eta>0$. To further simplify the results, we specify the choices of $\rho$ and $\eta$ to derive Corollary~\ref{co:epr-polyak}.
\begin{corollary}[Excess population risk of \eqref{alg:sgd-polyak}]\label{co:epr-polyak}
Consider the same setting as in Theorem \ref{thm:epr-polyak} and $\beta\in [0,1)$.
\wzm{
For any $t\gtrsim n/(1-\beta)$, the following holds:
\begin{itemize}
    \item[1.] If $L(\bw^*) \ge \frac{1}{n}$, choose $\rho = \sqrt{n L(\bw^*)}$ and
    $\eta\asymp\frac{(1-\beta)n}{t\sqrt{n L(\bw^*)}}$ to obtain
    \[
    \ebb[L(\bar{\bw}_t)-L(\bw^*)]\lesssim \sqrt{\frac{L(\bw^*)}{n}}.
    \]
    \item[2.] If $L(\bw^*) < \frac{1}{n}$, choose $\rho = 1$ and
    $\eta\asymp\frac{(1-\beta)n}{t}$ to obtain
    \[
    \ebb[L(\bar{\bw}_t)-L(\bw^*)]\lesssim \frac{1}{n}.
    \]
\end{itemize}
}
\end{corollary}

The following theorem provides analogous results for SGD with Nesterov's momentum.
\begin{theorem}[Excess population risk of \eqref{alg:sgd-nesterov}]\label{thm:epr-nesterov}
    Assume that for any $\bz\in\zcal$, the map $\bw\mapsto\ell(\bw;\bz)$ is nonnegative, convex, and $\alpha$-smooth.
    Let $\{\bw_t\}_{t\ge 1}$ be the sequence produced by SGD with Nesterov's momentum \eqref{alg:sgd-nesterov} with $\beta\in(0,1)$ applied to $S$ and denote $\bar{\bw}_t:=\frac{1}{t}\sum_{k=1}^t \bw_k$. Suppose the step size $\gamma>0$ is chosen such that
    \[
    \gamma\leq\frac{2(1-\beta)^2}{\alpha(3+6\beta+5\beta^2-2\beta^3)}.
    \]
    Then for any $t\ge 1$ and $\rho>0$, we have
    \begin{multline*}
\e[L(\bar{\bw}_t)]-L(\bw^*)\lesssim \frac{1-\beta}{\gamma t}+ \frac{\gamma}{(1-\beta)^2} L(\bw^*) + \frac{1}{\rho}\Big(L(\bw^*)+\frac{1-\beta}{\gamma t}+ \frac{\gamma}{(1-\beta)^2} L(\bw^*)\Big) \\
+ \frac{1+\rho}{n} \frac{\gamma^2}{(1-\beta)^2} \Big(\frac{1}{1-\beta}+\frac{t}{n}\Big) \Big(t L(\bw^*)+\frac{1-\beta}{\gamma}+ \frac{\gamma t}{(1-\beta)^2} L(\bw^*)\Big).
\end{multline*}
\end{theorem}
Note that the EPR bound established in Theorem \ref{thm:epr-nesterov} has the same form as the one in Theorem \ref{thm:epr-polyak}, except for $\eta$ replaced by $\gamma$. Therefore, by imitating the proof of Corollary \ref{co:epr-polyak}, we can similarly obtain the following results for \eqref{alg:sgd-nesterov}. We omit its proof for brevity.
\begin{corollary}[Excess population risk of \eqref{alg:sgd-nesterov}]\label{co:epr-nesterov}
Consider the same setting as in Theorem \ref{thm:epr-nesterov} and $\beta\in (0,1)$.
\wzm{
For any $t\gtrsim n/(1-\beta)$, the following hold:
\begin{itemize}
    \item[1.] If $L(\bw^*) \ge \frac{1}{n}$, choose $\rho = \sqrt{n L(\bw^*)}$ and
    $\gamma\asymp\frac{(1-\beta)n}{t\sqrt{n L(\bw^*)}}$ to obtain
    \[
    \ebb[L(\bar{\bw}_t)-L(\bw^*)]\lesssim \sqrt{\frac{L(\bw^*)}{n}}.
    \]
    \item[2.] If $L(\bw^*) < \frac{1}{n}$, choose $\rho = 1$ and
    $\gamma\asymp\frac{(1-\beta)n}{t}$ to obtain
    \[
    \ebb[L(\bar{\bw}_t)-L(\bw^*)]\lesssim \frac{1}{n}.
    \]
\end{itemize}
}
\end{corollary}
\begin{remark}\normalfont
    Corollaries~\ref{co:epr-polyak} and~\ref{co:epr-nesterov} specify conditions for selecting the parameters to obtain EPR bounds of order $O(\sqrt{L(\bw^*)/n})$ in the convex setting for both \eqref{alg:sgd-polyak} and \eqref{alg:sgd-nesterov}, \wzm{covering arbitrarily many passes over the data.}
    In particular, the resulting $O(1/\sqrt{n})$ bound is minimax optimal in the general case \citep{agarwal2009information}. Thus, our stability and optimization analysis yield optimal EPR bounds.
    Moreover, under the low-noise condition $L(\bw^*)<1/n$, we obtain an improved bound of order $O(1/n)$. This implies that when there exists a model with vanishing risk, our analysis leverages this property to demonstrate enhanced generalization performance. This perspective is closely aligned with the optimistic stability analysis developed for SGD in~\citet{lei2020fine}.
\end{remark}

\section{Technical Proofs in Section \ref{sec:gen-opt} and \ref{sec:epr}}\label{sec:proofs}
\subsection{Proofs of Stability Bounds}\label{sec:proof-stab}
In this subsection, we provide the omitted proofs of the results established in Section \ref{sec:stab}. We start by proving Lemma \ref{lem:ip-w-m}.

\medskip\noindent\textbf{Lemma~\ref{lem:ip-w-m}} (Restated). \textit{Let $S, S^{(i)}$ be defined in Definition~\ref{def:model-stable}. Let $\{\bw_t\}_{t\ge 1}$ and $\{\bw_t^{(i)}\}_{t\ge 1}$ be the sequences produced by the generalized SGDM Eq.~\eqref{eq:m-update}-\eqref{eq:w-update} with $\beta\in [0,1)$ applied to $S$ and $S^{(i)}$, respectively. Then for any $t\ge 1$ and $\delta>0$, we have
\begin{align*}
    \|\bfm_t-\bfm_t^{(i)}\|^2 \leq &(1+1/\delta) \sum_{k=1}^t \big((1+\delta)\beta^2\big)^{t-k}\|\bg_k-\bg_k^{(i)}\|^2,\\
    \langle \bw_t-\bw_t^{(i)}, \bfm_t-\bfm_t^{(i)}\rangle \ge & -(1+1 / \delta)\Big(\frac{\gamma}{2}+\eta\Big) \sum_{k=1}^{t-1}\Big(\sum_{l=k}^{t-1} \beta^{t-l}((1+\delta) \beta^2)^{l-k}\Big)\|\bg_k-\bg_k^{(i)}\|^2\\
    &-\frac{\gamma}{2} \sum_{k=1}^{t-1} \beta^{t-k}\|\bg_k-\bg_k^{(i)}\|^2 + \sum_{k=1}^t \beta^{t-k}\langle \bw_k-\bw_k^{(i)}, \bg_k-\bg_k^{(i)}\rangle.
\end{align*}}
\medskip

\begin{proof}[of Lemma \ref{lem:ip-w-m}]
Firstly, note that $\bfm_1=\bg_1$ and $\bfm_1^{(i)}=\bg_1^{(i)}$ due to $\bfm_0=\bfm_0^{(i)}=\mathbf{0}$. Hence Eq.~\eqref{eq:m-square} and \eqref{eq:cvx-stab-5} hold for $t=1$. Now we assume $t\ge 2$ in what follows.
Note from the update rule Eq.~\eqref{eq:m-update} that for any $\delta>0$, we have
\[
\|\bfm_t-\bfm_t^{(i)}\|^2 \leq \beta^2(1+\delta)\|\bfm_{t-1}-\bfm_{t-1}^{(i)}\|^2+(1+1/\delta)\|\bg_t-\bg_t^{(i)}\|^2,
\]
where we have used $(a+b)^2\leq (1+\delta)a^2+(1+1/\delta)b^2$.
Applying the above inequality recursively and noting that $\bfm_0=\bfm_0^{(i)}$ gives the bound Eq.~\eqref{eq:m-square}.

Next, we estimate the inner product $\langle \bw_t-\bw_t^{(i)}, \bfm_t-\bfm_t^{(i)}\rangle$. From Eq.~\eqref{eq:m-update} and Eq.~\eqref{eq:cvx-stab-n1}, together with the definitions of $G_t$ and $G_t^{(i)}$ in Eq.~\eqref{eq:Gt-def}, we obtain
\begin{align*}
&\langle \bw_t-\bw_t^{(i)}, \bfm_t-\bfm_t^{(i)}\rangle\\
=&\beta\langle \bw_t-\bw_t^{(i)}, \bfm_{t-1}-\bfm_{t-1}^{(i)}\rangle+\langle \bw_t-\bw_t^{(i)}, \bg_t-\bg_t^{(i)}\rangle \\
= & \beta\langle G_{t-1}-G_{t-1}^{(i)}, \bfm_{t-1}-\bfm_{t-1}^{(i)}\rangle-\beta \eta\|\bfm_{t-1}-\bfm_{t-1}^{(i)}\|^2+\langle \bw_t-\bw_t^{(i)}, \bg_t-\bg_t^{(i)}\rangle \\
= & \beta\langle \bw_{t-1}-\bw_{t-1}^{(i)}, \bfm_{t-1}-\bfm_{t-1}^{(i)}\rangle-\beta \gamma\langle \bg_{t-1}-\bg_{t-1}^{(i)}, \bfm_{t-1}-\bfm_{t-1}^{(i)}\rangle\\
&-\beta \eta\|\bfm_{t-1}-\bfm_{t-1}^{(i)}\|^2+\langle \bw_t-\bw_t^{(i)}, \bg_t-\bg_t^{(i)}\rangle.
\end{align*}
Applying Young's inequality to the term $-\beta \gamma\langle \bg_{t-1}-\bg_{t-1}^{(i)}, \bfm_{t-1}-\bfm_{t-1}^{(i)}\rangle$ further gives
\begin{multline*}
    \langle \bw_t-\bw_t^{(i)}, \bfm_t-\bfm_t^{(i)}\rangle
    \ge  \beta\langle \bw_{t-1}-\bw_{t-1}^{(i)}, \bfm_{t-1}-\bfm_{t-1}^{(i)}\rangle\\-\frac{\beta\gamma}{2}\|\bg_{t-1}-\bg_{t-1}^{(i)}\|^2-\Big(\frac{\beta\gamma}{2}+\beta \eta\Big)\|\bfm_{t-1}-\bfm_{t-1}^{(i)}\|^2+\langle \bw_t-\bw_t^{(i)}, \bg_t-\bg_t^{(i)}\rangle.
\end{multline*}
We apply the above inequality recursively and use $\bw_1-\bw_1^{(i)}=0$ to get
\begin{multline}
    \langle \bw_t-\bw_t^{(i)}, \bfm_t-\bfm_t^{(i)}\rangle \geq -\frac{\beta\gamma}{2} \sum_{k=1}^{t-1} \beta^{t-1-k}\|\bg_k-\bg_k^{(i)}\|^2\\
    -\Big(\frac{\beta\gamma}{2}+\beta \eta\Big) \sum_{k=1}^{t-1}\beta^{t-1-k}\|\bfm_k-\bfm_k^{(i)}\|^2+\sum_{k=1}^t \beta^{t-k}\langle \bw_k-\bw_k^{(i)}, \bg_k-\bg_k^{(i)}\rangle.\label{eq:cvx-stab-4}
\end{multline}
Applying the bound Eq.~\eqref{eq:m-square} for each $\|\bfm_k-\bfm_k^{(i)}\|^2$ in Eq.~\eqref{eq:cvx-stab-4} leads to
\begin{align*}
\langle \bw_t-\bw_t^{(i)}, \bfm_t-\bfm_t^{(i)}\rangle \ge & -\Big(\frac{\gamma}{2}+\eta\Big) \sum_{k=1}^{t-1} \beta^{t-k}(1+1 / \delta) \sum_{l=1}^k((1+\delta) \beta^2)^{k-l}\|\bg_l-\bg_l^{(i)}\|^2 \notag\\
& -\frac{\gamma}{2} \sum_{k=1}^{t-1} \beta^{t-k}\|\bg_k-\bg_k^{(i)}\|^2+\sum_{k=1}^t \beta^{t-k}\langle \bw_k-\bw_k^{(i)}, \bg_k-\bg_k^{(i)}\rangle \notag\\
= & -(1+1 / \delta)\Big(\frac{\gamma}{2}+\eta\Big) \sum_{l=1}^{t-1}\Big(\sum_{k=l}^{t-1} \beta^{t-k}((1+\delta) \beta^2)^{k-l}\Big)\|\bg_l-\bg_l^{(i)}\|^2 \notag\\
&-\frac{\gamma}{2} \sum_{k=1}^{t-1} \beta^{t-k}\|\bg_k-\bg_k^{(i)}\|^2 + \sum_{k=1}^t \beta^{t-k}\langle \bw_k-\bw_k^{(i)}, \bg_k-\bg_k^{(i)}\rangle \notag\\
= & -(1+1 / \delta)\Big(\frac{\gamma}{2}+\eta\Big) \sum_{k=1}^{t-1}\Big(\sum_{l=k}^{t-1} \beta^{t-l}((1+\delta) \beta^2)^{l-k}\Big)\|\bg_k-\bg_k^{(i)}\|^2 \notag\\
&-\frac{\gamma}{2} \sum_{k=1}^{t-1} \beta^{t-k}\|\bg_k-\bg_k^{(i)}\|^2 + \sum_{k=1}^t \beta^{t-k}\langle \bw_k-\bw_k^{(i)}, \bg_k-\bg_k^{(i)}\rangle,
\end{align*}
where the first equality follows by exchanging the order of summations, and the second one holds from exchanging the notations $k$ and $l$ in the double summation term. The proof is now complete.
\end{proof}

Then we proceed to provide the proof of Theorem \ref{thm:stab-cvx}, which establishes on-average model stability bounds for the generalized SGDM.

\medskip\noindent\textbf{Theorem~\ref{thm:stab-cvx}} (Restated). \textit{Let $S, S^{(i)}$ be defined in Definition~\ref{def:model-stable}. Let $\{\bw_t\}_{t\ge 1}$ and $\{\bw_t^{(i)}\}_{t\ge 1}$ be the sequences produced by the generalized SGDM Eq.~\eqref{eq:m-update}-\eqref{eq:w-update} with $\beta\in [0,1)$ applied to $S$ and $S^{(i)}$, respectively.
Assume that for any $\bz\in\zcal$, the map $\bw\mapsto\ell(\bw;\bz)$ is nonnegative, convex, and $\alpha$-smooth. If the step sizes $\gamma\ge 0$ and $\eta>0$ are chosen such that
\[
    \frac{(1+\beta)(3-\beta)}{(1-\beta)^2}\eta+\frac{\beta^2+3}{2(1-\beta)^2}\gamma\leq \frac{1}{\alpha},
\]
then for any $t\ge 1$, we have
\begin{align*}
&\ebb[\|\bw_{t+1}-\bw_{t+1}^{(i)}\|^2]\\
\leq & \Big(\frac{2(1+\beta)^2}{(1-\beta)^3}\eta^2+2\gamma^2+\frac{\beta(\beta^2+3)}{(1-\beta)^3}\gamma\eta+\frac{(\gamma+\eta)((1-\beta)\gamma+\eta)t}{(1-\beta)^2 n}\Big)\frac{8 \alpha e}{n} \sum_{k=1}^t \ebb[L_S(\bw_k)],
\end{align*}
where $e$ denotes Euler's number.}
\medskip

\begin{proof}[of Theorem \ref{thm:stab-cvx}]
We start by noting from Eq.~\eqref{eq:cvx-stab-n1} that
\begin{equation}\label{eq:cvx-stab-1}
\|\bw_{t+1}-\bw_{t+1}^{(i)}\|^2=\|G_t-G_t^{(i)}\|^2+\eta^2\|\bfm_t-\bfm_t^{(i)}\|^2-2 \eta\langle G_t-G_t^{(i)}, \bfm_t-\bfm_t^{(i)}\rangle.
\end{equation}
From the definition of $G_t$ and $G_t^{(i)}$ in Eq.~\eqref{eq:Gt-def}, we know that
\begin{equation}\label{eq:cvx-stab-G}
    \|G_t-G_t^{(i)}\|^2=\|\bw_t-\bw_t^{(i)}\|^2+\gamma^2\|\bg_t-\bg_t^{(i)}\|^2-2 \gamma\langle \bw_t-\bw_t^{(i)}, \bg_t-\bg_t^{(i)}\rangle
\end{equation}
and
\begin{align}
  &-2\eta\langle G_t-G_t^{(i)}, \bfm_t-\bfm_t^{(i)}\rangle \notag\\
  = &-2 \eta\langle \bw_t-\bw_t^{(i)}, \bfm_t-\bfm_t^{(i)}\rangle+2 \gamma \eta\langle \bg_t-\bg_t^{(i)}, \bfm_t-\bfm_t^{(i)}\rangle\notag\\
  \leq & -2 \eta\langle \bw_t-\bw_t^{(i)}, \bfm_t-\bfm_t^{(i)}\rangle + \eta^2\|\bfm_t-\bfm_t^{(i)}\|^2+\gamma^2\|\bg_t-\bg_t^{(i)}\|^2.\label{eq:cvx-stab-2}
\end{align}
Plugging Eq.~\eqref{eq:cvx-stab-G} and Eq.~\eqref{eq:cvx-stab-2} into Eq.~\eqref{eq:cvx-stab-1} gives
\begin{multline}\label{eq:cvx-stab-3}
    \|\bw_{t+1}-\bw_{t+1}^{(i)}\|^2\leq \|\bw_t-\bw_t^{(i)}\|^2+2\eta^2\|\bfm_t-\bfm_t^{(i)}\|^2+2\gamma^2\|\bg_t-\bg_t^{(i)}\|^2\\-2 \gamma\langle \bw_t-\bw_t^{(i)}, \bg_t-\bg_t^{(i)}\rangle-2 \eta\langle \bw_t-\bw_t^{(i)}, \bfm_t-\bfm_t^{(i)}\rangle.
\end{multline}

We first consider the simple case $\beta=0$, which implies that $\bfm_t=\bg_t$ and $\bfm_t^{(i)}=\bg_t^{(i)}$. Then, it follows from Eq.~\eqref{eq:cvx-stab-3} that
\begin{equation}\label{eq:stab-sgd-1}
    \|\bw_{t+1}-\bw_{t+1}^{(i)}\|^2\leq \|\bw_t-\bw_t^{(i)}\|^2+2(\gamma^2+\eta^2)\|\bg_t-\bg_t^{(i)}\|^2-2 (\gamma+\eta)\langle \bw_t-\bw_t^{(i)}, \bg_t-\bg_t^{(i)}\rangle.
\end{equation}
Next, we bound $\langle \bw_t-\bw_t^{(i)}, \bg_t-\bg_t^{(i)}\rangle$. Recall that $\bg_t := \nabla\ell(\bw_t;\bz_{i_t})$ and $\bg_t^{(i)} := \nabla\ell(\bw_t^{(i)};\bz_{i_t}^{(i)})$, where $\bz_{i_t}^{(i)}$ denotes the $i_t$-th element from $S^{(i)}$. Let us first consider the case $i_t\neq i$ (with probability $1-1/n$), which implies $\bz_{i_t}^{(i)}=\bz_{i_t}$ and hence $\bg_t^{(i)}=\nabla \ell(\bw_t^{(i)};\bz_{i_t})$. Then by co-coercivity of $\ell$ (Lemma~\ref{lem:coercivity}), we have
\[
\langle \bw_t-\bw_t^{(i)}, \bg_t-\bg_t^{(i)}\rangle \ge \frac{1}{\alpha}\|\bg_t-\bg_t^{(i)}\|^2.
\]
Otherwise if $i_t=i$ (with probability $1/n$), then $\bz_{i_t}=\bz_i$ and $\bz_{i_t}^{(i)}=\bz_i^{\prime}$. By Young's inequality, for any $p>0$ we have (note $\gamma+\eta>0$)
\[
-2 \langle \bw_t-\bw_t^{(i)}, \bg_t-\bg_t^{(i)}\rangle \leq \frac{p}{\gamma+\eta}\|\bw_t-\bw_t^{(i)}\|^2+\frac{\gamma+\eta}{p}\|\nabla\ell(\bw_t; \bz_i)-\nabla\ell(\bw_t^{(i)} ; \bz_i')\|^2.
\]
Combining the above two cases, we obtain
\begin{align}
-2 \ebb[\langle \bw_t-\bw_t^{(i)}, \bg_t-\bg_t^{(i)}\rangle]= & -2 \ebb\big[\langle \bw_t-\bw_t^{(i)}, \bg_t-\bg_t^{(i)}\rangle|i_t\neq i\big]\text{Pr}\{i_t\neq i\}\notag\\
&-2 \ebb\big[\langle \bw_t-\bw_t^{(i)}, \bg_t-\bg_t^{(i)}\rangle|i_t= i\big]\text{Pr}\{i_t= i\}\notag\\
\leq & -\frac{2}{\alpha}\ebb\big[\|\bg_t-\bg_t^{(i)}\|^2|i_t\neq i\big]\frac{n-1}{n}+\frac{p}{n(\gamma+\eta)}\ebb\big[\|\bw_t-\bw_t^{(i)}\|^2\big]\notag\\
&+\frac{\gamma+\eta}{np}\ebb\big[\|\nabla\ell(\bw_t; \bz_i)-\nabla\ell(\bw_t^{(i)} ; \bz_i')\|^2\big].\label{eq:cvx-stab-n2}
\end{align}
Similarly, we also know that
\begin{equation}\label{eq:cvx-stab-n3}
  \ebb\big[\|\bg_t-\bg_t^{(i)}\|^2\big]=\frac{n-1}{n}\ebb\big[\|\bg_t-\bg_t^{(i)}\|^2|i_t\neq i\big]+\frac{1}{n}\ebb\big[\|\nabla\ell(\bw_t; \bz_i)-\nabla\ell(\bw_t^{(i)} ; \bz_i')\|^2\big].
\end{equation}
Taking expectations on both sides of Eq.~\eqref{eq:stab-sgd-1} and applying Eq.~\eqref{eq:cvx-stab-n2} and \eqref{eq:cvx-stab-n3} leads to
\begin{multline}\label{eq:stab-sgd-2}
    \ebb\big[\|\bw_{t+1}-\bw_{t+1}^{(i)}\|^2\big]\leq \Big(\frac{(\gamma+\eta)^2}{np}+\frac{2(\gamma^2+\eta^2)}{n}\Big)\ebb\big[\|\nabla\ell(\bw_t; \bz_i)-\nabla\ell(\bw_t^{(i)} ; \bz_i')\|^2\big]\\+2\Big(\gamma^2+\eta^2-\frac{\gamma+\eta}{\alpha}\Big)\ebb\big[\|\bg_t-\bg_t^{(i)}\|^2|i_t\neq i\big]\cdot\frac{n-1}{n}+\Big(1+\frac{p}{n}\Big)\ebb\big[\|\bw_t-\bw_t^{(i)}\|^2\big].
\end{multline}
Plugging $\beta=0$ in the step size condition Eq.~\eqref{eq:stab-lr} gives
\[
3\eta+\frac{3}{2}\gamma\leq \frac{1}{\alpha}
\Longrightarrow \eta-1/\alpha\leq 0\;\;\text{and}\;\;\gamma\leq 1/\alpha,
\]
which further implies that
\begin{equation}\label{eq:stab-sgd-3}
    \gamma^2 + \eta^2- \frac{\gamma+\eta}{\alpha} =\gamma(\gamma-1/\alpha)+\eta (\eta - 1/\alpha)\leq 0.
\end{equation}
By the elementary inequality $(a+b)^2\leq 2a^2+2b^2$ and the self-bounding property of $\ell$ (Lemma~\ref{lem:self-bounding}), we obtain
\begin{align}
    \e[\|\nabla \ell(\bw_t; \bz_{i})-\nabla \ell(\bw_t^{(i)}; \bz_i^{\prime})\|^2] &\leq \e\big[2\|\nabla \ell(\bw_t; \bz_i)\|^2+2 \| \nabla\ell(\bw_t^{(i)}; \bz_i^{\prime}) \|^2\big]\notag\\
    &\leq 4 \alpha\e[\ell(\bw_t; \bz_i)+\ell(\bw_t^{(i)}; \bz_i^{\prime})]\notag\\
    &=8\alpha\ebb[\ell(\bw_t;\bz_{i})]= 8\alpha\ebb[L_S(\bw_t)]\label{eq:cvx-stab-loss-w},
\end{align}
where the last equality holds due to the symmetry of the algorithm
\[
\ebb[\ell(\bw_t;\bz_{i})] = \frac{1}{n}\sum_{j=1}^n \ebb[\ell(\bw_t;\bz_{j})] = \ebb[L_S(\bw_t)].
\]
Therefore, applying Eq.~\eqref{eq:stab-sgd-3} and Eq.~\eqref{eq:cvx-stab-loss-w} to Eq.~\eqref{eq:stab-sgd-2} leads to
\[
    \ebb\big[\|\bw_{t+1}-\bw_{t+1}^{(i)}\|^2\big]\leq 8\alpha\Big(\frac{(\gamma+\eta)^2}{np}+\frac{2(\gamma^2+\eta^2)}{n}\Big)\ebb[L_S(\bw_t)]+\Big(1+\frac{p}{n}\Big)\ebb\big[\|\bw_t-\bw_t^{(i)}\|^2\big].
\]
Applying the above inequality recursively and noting that $\bw_1=\bw_1^{(i)}$ further gives
\begin{equation}\label{eq:stab-sgd-4}
    \ebb\big[\|\bw_{t+1}-\bw_{t+1}^{(i)}\|^2\big]\leq 8\alpha\Big(\frac{(\gamma+\eta)^2}{np}+\frac{2(\gamma^2+\eta^2)}{n}\Big)\sum_{k=1}^t\Big(1+\frac{p}{n}\Big)^{t-k}\ebb[L_S(\bw_k)].
\end{equation}
Let $p=n/t$, then we have
\[
\Big(1+\frac{p}{n}\Big)^{t-k}\leq\Big(1+\frac{p}{n}\Big)^t = \Big(1+\frac{1}{t}\Big)^t\leq e.
\]
Hence it follows from Eq.~\eqref{eq:stab-sgd-4} that
\[
    \ebb\big[\|\bw_{t+1}-\bw_{t+1}^{(i)}\|^2\big]\leq \Big(\frac{t(\gamma+\eta)^2}{n}+2(\gamma^2+\eta^2)\Big)\frac{8\alpha e}{n}\sum_{k=1}^t\ebb[L_S(\bw_k)],
\]
which is exactly the stability bound Eq.~\eqref{eq:cvx-stab-sgdm} with $\beta=0$.

We proceed to consider $0<\beta<1$ and start from Eq.~\eqref{eq:cvx-stab-3}.
For notational convenience, we define
\[
C_{t, k}=\frac{\gamma}{2} \beta^{t-k}+(1+1 / \delta)\Big(\frac{\gamma}{2}+\eta\Big) \sum_{l=k}^{t-1} \beta^{t-l}((1+\delta) \beta^2)^{l-k},\quad \forall 1\leq k\leq t-1.
\]
Then it follows from Eq.~\eqref{eq:cvx-stab-5} in Lemma \ref{lem:ip-w-m} that
\begin{equation}\label{eq:cvx-stab-6}
    \langle \bw_t-\bw_t^{(i)}, \bfm_t-\bfm_t^{(i)}\rangle \ge -\sum_{k=1}^{t-1} C_{t, k}\|\bg_k-\bg_k^{(i)}\|^2+\sum_{k=1}^t \beta^{t-k}\langle \bw_k-\bw_k^{(i)}, \bg_k-\bg_k^{(i)}\rangle.
\end{equation}
Plugging Eq.~\eqref{eq:m-square} and Eq.~\eqref{eq:cvx-stab-6} into Eq.~\eqref{eq:cvx-stab-3}, we obtain
\begin{multline}\label{eq:cvx-stab-7}
\|\bw_{t+1}-\bw_{t+1}^{(i)}\|^2 \leq \|\bw_t-\bw_t^{(i)}\|^2+2(1+1 / \delta) \eta^2 \sum_{k=1}^t((1+\delta) \beta^2)^{t-k}\|\bg_k-\bg_k^{(i)}\|^2+2\gamma^2\|\bg_t-\bg_t^{(i)}\|^2 \\
+2 \eta \sum_{k=1}^{t-1} C_{t, k}\|\bg_k-\bg_k^{(i)}\|^2-2 \eta \sum_{k=1}^t \beta^{t-k}\langle \bw_k-\bw_k^{(i)}, \bg_k-\bg_k^{(i)}\rangle-2 \gamma\langle \bw_t-\bw_t^{(i)}, \bg_t-\bg_t^{(i)}\rangle.
\end{multline}
For notational convenience, define two sequences $\{A_{t, k}\}$ and $\{B_{t, k}\}$ via
\begin{equation*}
A_{t,k} = \begin{cases}
(1+1/\delta)\eta^2 + \gamma^2, & \text{if } k = t \\
(1+1/\delta)\eta^2((1+\delta)\beta^2)^{t-k} + \eta C_{t,k}, & \text{if }k \leq t-1
\end{cases}, \quad
B_{t,k} = \begin{cases}
\eta + \gamma, & \text{if }k = t, \\
\eta \beta^{t-k}, & \text{if }k \leq t-1.
\end{cases}
\end{equation*}
Then we can rewrite Eq.~\eqref{eq:cvx-stab-7} as
\begin{equation}\label{eq:cvx-stab-8}
    \|\bw_{t+1}-\bw_{t+1}^{(i)}\|^2 \leq \|\bw_t-\bw_t^{(i)}\|^2+2\sum_{k=1}^t A_{t, k} \|\bg_k-\bg_k^{(i)}\|^2 - 2\sum_{k=1}^t B_{t, k}\langle \bw_k-\bw_k^{(i)}, \bg_k-\bg_k^{(i)}\rangle.
\end{equation}
Due to the independence among the indexes $\{i_k\}$, we know that Eq.~\eqref{eq:cvx-stab-n2} and Eq.~\eqref{eq:cvx-stab-n3} hold with $t$ replaced by any $k\in [t]$. Applying the resulting inequalities on Eq.~\eqref{eq:cvx-stab-8} gives
\begin{align}
\ebb[\|\bw_{t+1}-\bw_{t+1}^{(i)}\|^2] \leq &\ebb[\|\bw_t-\bw_t^{(i)}\|^2]+2 \sum_{k=1}^t(A_{t, k}-B_{t, k} / \alpha) \ebb[\|\bg_k-\bg_k^{(i)}\|^2 \mid i_k \neq i]\cdot \frac{n-1}{n} \notag\\
&+\frac{1}{n} \sum_{k=1}^t(2 A_{j,k}+(\gamma+\eta) B_{j,k}/p) \ebb[\|\nabla\ell(\bw_k; \bz_i)-\nabla\ell(\bw_k^{(i)} ; \bz_i')\|^2]\notag\\
&+\frac{p}{n(\gamma+\eta)} \sum_{k=1}^t B_{t, k} \ebb[\|\bw_k-\bw_k^{(i)}\|^2]\label{eq:cvx-stab-9}.
\end{align}
Next we estimate $A_{t,k}$ for $k\leq t-1$. To this end, note from the definition of $C_{t,k}$ that
\begin{align*}
C_{t,k} & = \frac{\gamma}{2} \beta^{t-k}+(1+1 / \delta)\Big(\frac{\gamma}{2}+\eta\Big) \sum_{l=k}^{t-1} \beta^{t-l}((1+\delta) \beta^2)^{l-k}\\
& =\frac{\gamma}{2} \beta^{t-k}+(1+1 / \delta)\Big(\frac{\gamma}{2}+\eta\Big) \sum_{l=k}^{t-1} (1+\delta)^{l-k} \beta^{t+l-2 k}\\
& =\frac{\gamma}{2} \beta^{t-k}+(1+1 / \delta)\Big(\frac{\gamma}{2}+\eta\Big)(1+\delta)^{-k} \beta^{t-2 k} \sum_{l=k}^{t-1}((1+\delta) \beta)^l.
\end{align*}
Let $\delta:=\frac{1}{2}\big(\frac{1}{\beta}-1\big)>0$, which implies that
\begin{equation}\label{eq:cvx-stab-n4}
(1+\delta)\beta=(\beta +1)/2<1\quad\text{and}\quad 1+\frac{1}{\delta}=1+\frac{2\beta}{1-\beta}=\frac{1+\beta}{1-\beta}.
\end{equation}
Hence, we can further obtain
\begin{align}
C_{t, k} & \leq \frac{\gamma}{2} \beta^{t-k}+(1+1 / \delta)\Big(\frac{\gamma}{2}+\eta\Big)(1+\delta)^{-k} \beta^{t-2 k} \frac{((1+\delta) \beta)^k}{1-(1+\delta) \beta}\notag\\
&=\frac{\gamma}{2} \beta^{t-k}+(1+1 / \delta)\Big(\frac{\gamma}{2}+\eta\Big)\frac{\beta^{t-k}}{1-(1+\delta)\beta} \notag\\
&= \Big(\frac{\gamma}{2}+\frac{(1+1 / \delta)(\gamma/2+\eta)}{1-(1+\delta)\beta}\Big)\beta^{t-k}\notag\\
&= \Big(\frac{\gamma}{2}+\frac{(1+\beta)(\gamma+2\eta)}{(1-\beta)^2}\Big)\beta^{t-k}\label{eq:Ck-bound},
\end{align}
where we apply Eq.~\eqref{eq:cvx-stab-n4} to get the last equality (note $1-(1+\delta)\beta=\frac{1-\beta}{2}$).
Then it follows from the definition of $A_{t,k}$, Eq.~\eqref{eq:cvx-stab-n4}, and Eq.~\eqref{eq:Ck-bound} that for any $k\leq t-1$
\begin{align}
    A_{t,k} &=(1+1/\delta)\eta^2((1+\delta)\beta^2)^{t-k} + \eta C_{t,k}
    \leq \frac{1+\beta}{1-\beta}\eta^2\beta^{t-k}+\eta C_{t,k}\notag\\
    &\leq \Big(\frac{1+\beta}{1-\beta}\eta+\frac{\gamma}{2}+\frac{(1+\beta)(\gamma+2\eta)}{(1-\beta)^2}\Big)\eta\beta^{t-k}\notag\\
    &= \Big[\Big(\frac{1+\beta}{1-\beta}+\frac{2(1+\beta)}{(1-\beta)^2}\Big)\eta+\Big(\frac{1}{2}+\frac{1+\beta}{(1-\beta)^2}\Big)\gamma\Big]\eta\beta^{t-k}\notag\\
    &= \Big[\frac{(1+\beta)(3-\beta)}{(1-\beta)^2}\eta+\frac{\beta^2+3}{2(1-\beta)^2}\gamma\Big]\eta\beta^{t-k},\label{eq:Ak-bound}
\end{align}
where the first inequality holds by noting that $(1+\delta)\beta^2\leq\beta$.
Then it follows from the step size condition Eq.~\eqref{eq:stab-lr} and Eq.~\eqref{eq:Ak-bound} that
\begin{equation}\label{eq:cvx-A-neg-1}
    A_{t, k}-B_{t, k} / \alpha\leq 0,\quad \forall k\leq t-1.
\end{equation}
Additionally, since $\gamma, \eta\ge 0$, it follows from Eq.~\eqref{eq:stab-lr} that
\[
\frac{(1+\beta)(3-\beta)}{(1-\beta)^2}\eta\leq \frac{1}{\alpha},\quad \frac{\beta^2+3}{2(1-\beta)^2}\gamma\leq \frac{1}{\alpha},
\]
which together with $\beta\in(0,1)$ further implies
\begin{align*}
    &\frac{1+\beta}{1-\beta}\eta^2\leq\frac{1-\beta}{(3-\beta)\alpha}\eta\leq\frac{\eta}{2\alpha}\leq \frac{\eta}{\alpha},\\
    &\gamma^2\leq\frac{2(1-\beta)^2}{(\beta^2+3)\alpha}\gamma\leq \frac{2\gamma}{3\alpha}\leq\frac{\gamma}{\alpha}.
\end{align*}
Using the above two inequalities, we obtain from Eq.~\eqref{eq:cvx-stab-n4} and the definition of $A_{t,t}$ and $B_{t,t}$ that
\[
A_{t, t}-B_{t, t} / \alpha =  \frac{1+\beta}{1-\beta}\eta^2 + \gamma^2 - \frac{\eta+\gamma}{\alpha}\leq 0,
\]
which together with Eq.~\eqref{eq:cvx-A-neg-1} gives
\begin{equation}\label{eq:cvx-A-neg}
    A_{t, k}-B_{t, k} / \alpha\leq 0,\quad \forall k\in [t].
\end{equation}
Therefore, applying Eq.~\eqref{eq:cvx-A-neg} and Eq.~\eqref{eq:cvx-stab-loss-w} with $t$ replaced by $k\in[t]$ to Eq.~\eqref{eq:cvx-stab-9} leads to
\begin{multline}\label{eq:cvx-stab-10}
\ebb[\|\bw_{t+1}-\bw_{t+1}^{(i)}\|^2] \leq \ebb[\|\bw_t-\bw_t^{(i)}\|^2]+\frac{8\alpha}{n} \sum_{k=1}^t(2 A_{t, k}+(\gamma+\eta) B_{t, k} / p) \ebb[L_S(\bw_k)]\\+\frac{p}{n(\gamma+\eta)} \sum_{k=1}^t B_{t, k} \ebb[\|\bw_k-\bw_k^{(i)}\|^2].
\end{multline}
Define
\[
\xi_t :=\max _{k \in[t]} \ebb[\|\bw_k-\bw_k^{(i)}\|^2].
\]
Then it follows from Eq.~\eqref{eq:cvx-stab-10} that
\[
\ebb[\|\bw_{t+1}-\bw_{t+1}^{(i)}\|^2] \leq \Big(1+\frac{p}{n(\gamma+\eta)} \sum_{k=1}^t B_{t, k}\Big)\xi_t+\frac{8\alpha}{n} \sum_{k=1}^t(2 A_{t, k}+(\gamma+\eta) B_{t, k} / p) \ebb[L_S(\bw_k)].
\]
Moreover, note that $\sum_{k=1}^t B_{t,k}=\eta+\gamma+\eta\sum_{k=1}^{t-1}\beta^{t-k}$ and $\xi_{t+1}=\max\Big\{\xi_t,~\ebb[\|\bw_{t+1}-\bw_{t+1}^{(i)}\|^2]\Big\}$, we further get
\begin{equation}\label{eq:cvx-stab-xi}
\xi_{t+1} \leq\Big(1+\frac{p}{n} +\frac{p\eta}{n(\gamma+\eta)}\sum_{k=1}^{t-1} \beta^{t-k}\Big) \xi_t+\frac{8\alpha}{n} \sum_{k=1}^t(2A_{t, k}+(\gamma+\eta) B_{t, k} / p) \ebb[L_S(\bw_k)].
\end{equation}
Note that $\sum_{k=1}^{t-1} \beta^{t-k}$ is an increasing function in terms of $t$ and $\xi_1 = 0$ due to $\bw_1 = \bw_1^{(i)}$, thus applying Eq.~\eqref{eq:cvx-stab-xi} recursively leads to
\begin{align}
\xi_{t+1} & \leq \frac{8 \alpha}{n} \sum_{j=1}^t\Big(1+\frac{p}{n} +\frac{p\eta}{n(\gamma+\eta)}\sum_{k'=1}^{t-1} \beta^{t-k'}\Big)^{t-j}\Big[\sum_{k=1}^j(2 A_{j,k}+(\gamma+\eta) B_{j,k}/p) \ebb[L_S(\bw_k)]\Big]\notag \\
& = \frac{8 \alpha}{n} \sum_{k=1}^t \sum_{j=k}^t\Big(1+\frac{p}{n} +\frac{p\eta}{n(\gamma+\eta)}\sum_{k'=1}^{t-1} \beta^{t-k'}\Big)^{t-j}(2 A_{j,k}+(\gamma+\eta) B_{j,k}/p)\ebb[L_S(\bw_k)],\label{eq:cvx-stab-11}
\end{align}
where the equality holds by exchanging the order of the two summations. Let $p=(1-\beta)n/t$, then we have
\[
\frac{p\eta}{n(\gamma+\eta)} \sum_{k'=1}^{t-1} \beta^{t-k'}\leq\frac{p}{n}\sum_{k'=1}^{t-1} \beta^{t-k'}
=\frac{1-\beta}{t} \sum_{j=1}^{t-1} \beta^j=\frac{1-\beta}{t} \cdot \frac{\beta(1-\beta^{t-1})}{1-\beta} \leq\frac{\beta}{t},
\]
which further implies
\[
\Big(1+\frac{p}{n} +\frac{p\eta}{n(\gamma+\eta)}\sum_{k'=1}^{t-1} \beta^{t-k'}\Big)^{t-j}\leq \Big(1+\frac{p}{n} +\frac{p\eta}{n(\gamma+\eta)}\sum_{k'=1}^{t-1} \beta^{t-k'}\Big)^t\leq \Big(1+\frac{1}{t}\Big)^t\leq e,~\forall j\leq t.
\]
Hence it follows from Eq.~\eqref{eq:cvx-stab-11} that
\begin{equation}\label{eq:cvx-stab-12}
    \xi_{t+1} \leq \frac{8 \alpha e}{n} \sum_{k=1}^t\Big(\sum_{j=k}^t(2 A_{j,k}+(\gamma+\eta) B_{j,k}/p)\Big) \ebb[L_S(\bw_k)].
\end{equation}
Additionally, for any $k\in [t]$, note from Eq.~\eqref{eq:cvx-stab-n4}, Eq.~\eqref{eq:Ak-bound}, and $p=(1-\beta)n/t$ that
\begin{align}
&\sum_{j=k}^t(2 A_{j,k}+(\gamma+\eta) B_{j,k}/p) =\frac{2(1+\beta)}{1-\beta}\eta^2+2\gamma^2+\frac{1}{p}(\gamma+\eta)^2+\sum_{j=k+1}^t\Big(2 A_{j,k}+\frac{\eta^2+\gamma\eta}{p} \beta^{j-k}\Big)\notag\\
\leq &\frac{2(1+\beta)}{1-\beta}\eta^2+2\gamma^2+\frac{1}{p}(\gamma+\eta)^2+ \Big(\frac{2(1+\beta)(3-\beta)}{(1-\beta)^2}\eta^2+\frac{\beta^2+3}{(1-\beta)^2}\gamma\eta+\frac{\eta^2+\gamma\eta}{p}\Big)\sum_{j=k+1}^t\beta^{j-k}\notag\\
\leq &\frac{2(1+\beta)}{1-\beta}\eta^2+2\gamma^2+\frac{1}{p}(\gamma^2+2\gamma\eta+\eta^2)+ \Big(\frac{2(1+\beta)(3-\beta)}{(1-\beta)^2}\eta^2+\frac{\beta^2+3}{(1-\beta)^2}\gamma\eta+\frac{\eta^2+\gamma\eta}{p}\Big)\frac{\beta}{1-\beta}\notag\\
= &\Big(\frac{2(1+\beta)}{1-\beta}+\frac{2\beta(1+\beta)(3-\beta)}{(1-\beta)^3}\Big)\eta^2+2\gamma^2+\frac{\beta(\beta^2+3)}{(1-\beta)^3}\gamma\eta+\frac{\eta^2}{(1-\beta)p}+\frac{(2-\beta)\gamma\eta}{(1-\beta)p}+\frac{\gamma^2}{p}\notag\\
= &\Big(\frac{2(1+\beta)}{1-\beta}+\frac{2\beta(1+\beta)(3-\beta)}{(1-\beta)^3}\Big)\eta^2+2\gamma^2+\frac{\beta(\beta^2+3)}{(1-\beta)^3}\gamma\eta+\frac{(\gamma+\eta)((1-\beta)\gamma+\eta)}{(1-\beta)p}\notag\\
= &\frac{2(1+\beta)^2}{(1-\beta)^3}\eta^2+2\gamma^2+\frac{\beta(\beta^2+3)}{(1-\beta)^3}\gamma\eta+\frac{(\gamma+\eta)((1-\beta)\gamma+\eta)t}{(1-\beta)^2 n}\label{eq:cvx-stab-coef},
\end{align}
where we have used $\sum_{j=k+1}^{t}\beta^{j-k}=\sum_{j=1}^{t-k}\beta^j\leq\frac{\beta}{1-\beta}$ to obtain the last inequality.
Applying Eq.~\eqref{eq:cvx-stab-coef} in Eq.~\eqref{eq:cvx-stab-12} further implies
\begin{align*}
&\ebb[\|\bw_{t+1}-\bw_{t+1}^{(i)}\|^2] \leq\xi_{t+1} \\
\leq &\Big(\frac{2(1+\beta)^2}{(1-\beta)^3}\eta^2+2\gamma^2+\frac{\beta(\beta^2+3)}{(1-\beta)^3}\gamma\eta+\frac{(\gamma+\eta)((1-\beta)\gamma+\eta)t}{(1-\beta)^2 n}\Big)\frac{8 \alpha e}{n} \sum_{k=1}^t \ebb[L_S(\bw_k)],
\end{align*}
which completes the proof.
\end{proof}

\wzm{
\subsubsection{Time-varying momentum}
In this subsection, we consider time-varying momentum for~\eqref{alg:sgd-nesterov}. Starting from $\bu_0=\bw_1$, the method performs the following updates for $t\geq1$:
\begin{equation*}
\tag{SGD-N-TV}\label{alg:sgd-nesterov-tv}
    \begin{aligned}
        &\bu_t=\bw_t-\gamma\bg_t,\\
        &\bw_{t+1}=\bu_t+\beta_t(\bu_t-\bu_{t-1}),
    \end{aligned}
\end{equation*}
where $\{\beta_t\}_{t\geq1}\subset[0,1]$ are momentum parameters and $\gamma>0$ is a fixed step size.

Note that~\eqref{alg:sgd-nesterov-tv} can be rewritten in the form of Eqs.~\eqref{eq:m-update}-\eqref{eq:w-update}. Introduce $\bfm_0=\mathbf{0}$ and $\bfm_t:=-(\bu_t-\bu_{t-1})/\gamma$ for $t\geq1$. Since $\bu_0=\bw_1$, we have $\bfm_1=\bg_1$. For $t\geq2$, we obtain from $\bw_t=\bu_{t-1}+\beta_{t-1}(\bu_{t-1}-\bu_{t-2})$ that
\begin{align*}
\bfm_t
&=-\frac{1}{\gamma}(\bw_t-\gamma\bg_t-\bu_{t-1})\\
&=-\frac{\beta_{t-1}}{\gamma}(\bu_{t-1}-\bu_{t-2})+\bg_t
=\beta_{t-1}\bfm_{t-1}+\bg_t.
\end{align*}
On the other hand, the second equation in~\eqref{alg:sgd-nesterov-tv} gives $\bw_{t+1}=\bw_t-\gamma\bg_t-\gamma\beta_t\bfm_t$. Therefore, we can rewrite~\eqref{alg:sgd-nesterov-tv} as
\begin{align}
&\bfm_t=\beta_{t-1}\bfm_{t-1}+\bg_t,\label{eq:sgd-nes-1}\\
&\bw_{t+1}=\bw_t-\gamma\bg_t-\gamma\beta_t\bfm_t.\label{eq:sgd-nes-2}
\end{align}
Using the reformulation Eqs.~\eqref{eq:sgd-nes-1}-\eqref{eq:sgd-nes-2}, we now establish polynomial stability bounds for \eqref{alg:sgd-nesterov-tv} that apply to any $\{\beta_t\}_{t\geq1}\subset[0,1]$ and $\gamma>0$.
\begin{proposition}[Polynomial stability for~\eqref{alg:sgd-nesterov-tv}]\label{prop:stab-cvx-nesterov-tv-convex}
Let $S, S^{(i)}$ be defined in Definition~\ref{def:model-stable}. Let $\{\bw_t\}_{t\ge 1}$ and $\{\bw_t^{(i)}\}_{t\ge 1}$ be the sequences produced by~\eqref{alg:sgd-nesterov-tv} applied to $S$ and $S^{(i)}$, respectively.
Assume that for any $\bz\in\zcal$, the map $\bw\mapsto\ell(\bw;\bz)$ is nonnegative, convex, and $\alpha$-smooth.
Denote
\[
h_{t,k}:=1+\sum_{j=k}^t\prod_{r=k}^j\beta_r,\qquad a_{t,k}:=h_{t,k}(h_{t,k}-1),
\]
where an empty product is interpreted as one by convention. Then, for any $\gamma>0$ and $t\geq1$,
\begin{multline}\label{eq:cvx-stab-nesterov-tv-convex}
\frac{1}{n}\sum_{i=1}^n\ebb[\|\bw_{t+1}-\bw_{t+1}^{(i)}\|^2]
\leq\frac{8\gamma \min\{t,n\}}{n}\Big(a_{t,1}+\sum_{k=1}^th_{t,k}\Big)\ebb[L_S(\bw_1)]\\
+\frac{8\gamma \min\{t,n\}}{n}\sum_{k=1}^t
\big[(1+\alpha^2\gamma^2)a_{t,k+1}-(1-\alpha\gamma)h_{t,k}^2\big]\ebb[L_S(\bw_k)].
\end{multline}
In particular, we have
\begin{equation}\label{eq:nesterov-tv-convex-polynomial}
\frac{1}{n}\sum_{i=1}^n\ebb[\|\bw_{t+1}-\bw_{t+1}^{(i)}\|^2]
\leq\frac{8\gamma \min\{t,n\} t^2}{n}
\big(5+4\alpha\gamma+\alpha^2\gamma^2\big)
\sum_{k=1}^t\ebb[L_S(\bw_k)].
\end{equation}
\end{proposition}

\begin{proof}[of Proposition~\ref{prop:stab-cvx-nesterov-tv-convex}]
We first unroll Eq.~\eqref{eq:sgd-nes-1} to obtain
\begin{equation}\label{eq:momentum-unroll}
    \bfm_j=\sum_{k=1}^j\Big(\prod_{r=k}^{j-1}\beta_r\Big)\bg_k.
\end{equation}
Therefore, applying Eq.~\eqref{eq:sgd-nes-2} recursively and using the above identity gives
\begin{align}
\bw_{t+1}-\bw_1&=-\gamma\sum_{j=1}^t\bg_j
-\gamma\sum_{j=1}^t\beta_j \bfm_j\notag\\
&=-\gamma\sum_{j=1}^t\bg_j
-\gamma\sum_{j=1}^t\beta_j\sum_{k=1}^j
\Big(\prod_{r=k}^{j-1}\beta_r\Big)\bg_k\notag\\
&=-\gamma\sum_{k=1}^t
\Big(1+\sum_{j=k}^t\prod_{r=k}^j\beta_r\Big)\bg_k
=-\gamma\sum_{k=1}^th_{t,k}\bg_k,\label{eq:nesterov-tv-kernel}
\end{align}
where we exchanged the order of summation and absorbed $\beta_j$ into the product in the third equality.

For $1\leq l\leq k\leq t$, we can decompose $h_{t,l}$ as follows:
\begin{align*}
h_{t,l}&=1+\sum_{j=l}^{k-1}\prod_{r=l}^j\beta_r+\sum_{j=k}^t\prod_{r=l}^j\beta_r=h_{k-1,l}+\Big(\prod_{r=l}^{k-1}\beta_r\Big)\sum_{j=k}^t\prod_{r=k}^j\beta_r\\
&=h_{k-1,l}+(h_{t,k}-1)\prod_{r=l}^{k-1}\beta_r.
\end{align*}
Then it follows from Eq.~\eqref{eq:nesterov-tv-kernel} and the above identity that
\begin{align}
&\|\bw_{t+1}-\bw_1\|^2=\gamma^2\sum_{k=1}^t\sum_{j=1}^t h_{t,j}h_{t,k}\langle\bg_j,\bg_k\rangle\notag\\
=&\gamma^2\sum_{k=1}^th_{t,k}^2\|\bg_k\|^2+2\gamma^2\sum_{k=1}^th_{t,k}\sum_{j<k}h_{t,j}\langle\bg_j,\bg_k\rangle\notag\\
=&\gamma^2\sum_{k=1}^th_{t,k}^2\|\bg_k\|^2+2\gamma^2\sum_{k=1}^t h_{t,k}\sum_{j=1}^{k-1}\Big(h_{k-1,j}+(h_{t,k}-1)\prod_{r=j}^{k-1}\beta_r\Big)\langle\bg_j,\bg_k\rangle\notag\\
=&\gamma^2\sum_{k=1}^th_{t,k}^2\|\bg_k\|^2-2\gamma\sum_{k=1}^th_{t,k}\langle\bw_k-\bw_1,\bg_k\rangle+2\gamma^2\sum_{k=1}^ta_{t,k}\beta_{k-1}\langle\bfm_{k-1},\bg_k\rangle,\label{eq:wt_sq-expand}
\end{align}
where we used $\bw_{k}-\bw_1=-\gamma\sum_{j=1}^{k-1}h_{k-1,j}\bg_j$ from Eq.~\eqref{eq:nesterov-tv-kernel} and $\sum_{j=1}^{k-1}\big(\prod_{r=j}^{k-1}\beta_r\big)\bg_j=\beta_{k-1}\bfm_{k-1}$ from Eq.~\eqref{eq:momentum-unroll} to obtain the last equality.

To proceed, convexity of $L_S$ and $\bu_{k-1}-\bw_k=\gamma\beta_{k-1}\bfm_{k-1}$ gives
\begin{align*}
L_S(\bw_1)&\geq L_S(\bw_k)+\langle\nabla L_S(\bw_k),\bw_1-\bw_k\rangle,\\
L_S(\bu_{k-1})&\geq L_S(\bw_k)+
\gamma\beta_{k-1}\langle\nabla L_S(\bw_k),\bfm_{k-1}\rangle,
\end{align*}
which also holds at $k=1$ since $\bu_0=\bw_1$ and $\bfm_0=0$. Thus, taking conditional expectation $\ebb_k[\cdot]$ and noting $\ebb_k[\bg_k]=\nabla L_S(\bw_k)$ implies
\begin{align*}
&\ebb_k[\langle\bw_k-\bw_1,\bg_k\rangle]=\langle\bw_k-\bw_1,\nabla L_S(\bw_k)\rangle\geq L_S(\bw_k)-L_S(\bw_1),\\
&\gamma\beta_{k-1}\ebb_k[\langle\bfm_{k-1},\bg_k\rangle]\leq L_S(\bu_{k-1})-L_S(\bw_k).
\end{align*}
Taking expectations on both sides of Eq.~\eqref{eq:wt_sq-expand} and using the above two inequalities yield
\begin{align}
&\ebb[\|\bw_{t+1}-\bw_1\|^2]\notag\\
\leq&\gamma^2\sum_{k=1}^th_{t,k}^2\ebb[\|\bg_k\|^2]+2\gamma\sum_{k=1}^th_{t,k}\ebb[L_S(\bw_1)-L_S(\bw_k)]+2\gamma\sum_{k=1}^ta_{t,k}\ebb[L_S(\bu_{k-1})-L_S(\bw_k)]\notag\\
=&2\gamma \Big(a_{t,1}+\sum_{k=1}^th_{t,k}\Big)\ebb[L_S(\bw_1)]+\gamma^2\sum_{k=1}^th_{t,k}^2\ebb[\|\bg_k\|^2]+2\gamma\sum_{k=1}^t\ebb[a_{t,k+1}L_S(\bu_k)-h_{t,k}^2L_S(\bw_k)],\label{eq:wt-expand1}
\end{align}
where we use $h_{t,k}+a_{t,k}=h_{t,k}^2$ and the following reindexing due to $\bu_0=\bw_1$ and $a_{t,t+1}=0$:
\[
\sum_{k=1}^ta_{t,k}\ebb[L_S(\bu_{k-1})]=a_{t,1}\ebb[L_S(\bw_1)]+\sum_{k=1}^ta_{t,k+1}\ebb[L_S(\bu_k)].
\]
We next bound $\ebb[L_S(\bu_k)]$ using $\ebb[L_S(\bw_k)]$. Smoothness and $\bu_k=\bw_k-\gamma\bg_k$ from~\eqref{alg:sgd-nesterov-tv} imply
\[
L_S(\bu_k)\leq L_S(\bw_k)
-\gamma\langle\nabla L_S(\bw_k),\bg_k\rangle
+\frac{\alpha\gamma^2}{2}\|\bg_k\|^2.
\]
Taking conditional and then full expectations gives
\begin{align}
\ebb[L_S(\bu_k)]
&\leq\ebb[L_S(\bw_k)]-\gamma\ebb[\|\nabla L_S(\bw_k)\|^2]+\frac{\alpha\gamma^2}{2}\ebb[\|\bg_k\|^2]\notag\\
&\leq\ebb[L_S(\bw_k)]+\frac{\alpha\gamma^2}{2}\ebb[\|\bg_k\|^2].\label{eq:LS-uk-bd}
\end{align}
We also note the following from the self-bounding property:
\begin{equation}
    \label{eq:gk_sq-bd}
    \ebb[\|\bg_k\|^2]=\ebb\Big[\frac{1}{n}\sum_{j=1}^n\|\nabla\ell(\bw_k;\bz_j)\|^2\Big]\leq 2\alpha\ebb\Big[\frac{1}{n}\sum_{j=1}^n\ell(\bw_k;\bz_j)\Big]= 2\alpha\ebb[L_S(\bw_k)].
\end{equation}
Since $a_{t,k+1}\geq0$, substituting Eq.~\eqref{eq:LS-uk-bd} and Eq.~\eqref{eq:gk_sq-bd} into Eq.~\eqref{eq:wt-expand1} yield
\begin{align}
&\ebb[\|\bw_{t+1}-\bw_1\|^2]\notag\\
\leq& 2\gamma \Big(a_{t,1}+\sum_{k=1}^th_{t,k}\Big)\ebb[L_S(\bw_1)]+2\gamma\sum_{k=1}^t(a_{t,k+1}-h_{t,k}^2)\ebb[L_S(\bw_k)]+\gamma^2\sum_{k=1}^t(h_{t,k}^2+\alpha\gamma a_{t,k+1})\ebb[\|\bg_k\|^2]\notag\\
\leq&2\gamma \Big(a_{t,1}+\sum_{k=1}^th_{t,k}\Big)\ebb[L_S(\bw_1)]+2\gamma\sum_{k=1}^t\big[(1+\alpha^2\gamma^2)a_{t,k+1}-(1-\alpha\gamma)h_{t,k}^2\big]\ebb[L_S(\bw_k)].\label{eq:wt-expand2}
\end{align}

Let $\mathcal I_t:=\{i_1,\ldots,i_t\}$ be the set of distinct sampled indices and $J_{t,i}:=\ibb_{\{i\in\mathcal I_t\}}$ be the indicator variable for the event $i\in\mathcal I_t$. Then it follows that
\begin{equation}\label{eq:w-wi-1}
   \|\bw_{t+1}-\bw_{t+1}^{(i)}\|^2\leq 2J_{t,i}\big(\|\bw_{t+1}-\bw_1\|^2+\|\bw_{t+1}^{(i)}-\bw_1\|^2\big). 
\end{equation}
Moreover, we know from the symmetry between $\bz_i$ and $\bz_i'$ that
\begin{equation}\label{eq:J-symmetry}
\ebb[J_{t,i}\|\bw_{t+1}^{(i)}-\bw_1\|^2]
=\ebb[J_{t,i}\|\bw_{t+1}-\bw_1\|^2].
\end{equation}
Therefore, combining Eqs.~\eqref{eq:wt-expand2}, ~\eqref{eq:w-wi-1}, and~\eqref{eq:J-symmetry}, together with the observations that $\sum_{i=1}^{n}J_{t,i}=|\mathcal I_t|$ and $|\mathcal I_t|\leq \min\{t,n\}$, we obtain
\begin{align*}
&\frac{1}{n}\sum_{i=1}^n\ebb[\|\bw_{t+1}-\bw_{t+1}^{(i)}\|^2]\leq \frac{4}{n}\sum_{i=1}^n\ebb[J_{t,i}\|\bw_{t+1}-\bw_1\|^2]\\
=&\frac{4}{n}\ebb[|\mathcal I_t|\|\bw_{t+1}-\bw_1\|^2]\leq\frac{4\min\{t,n\}}{n}\ebb[\|\bw_{t+1}-\bw_1\|^2]\\
\leq&\frac{8\gamma\min\{t,n\}}{n}\Big(\Big(a_{t,1}+\sum_{k=1}^th_{t,k}\Big)\ebb[L_S(\bw_1)]+\sum_{k=1}^t\big[(1+\alpha^2\gamma^2)a_{t,k+1}-(1-\alpha\gamma)h_{t,k}^2\big]\ebb[L_S(\bw_k)]\Big),
\end{align*}
which proves Eq.~\eqref{eq:cvx-stab-nesterov-tv-convex}.

If $\beta_k\leq 1$ for all $k\geq1$, we have for any $k\ge 1$ that
\[
1\leq h_{t,k}\leq 1+\sum_{j=k}^t 1\leq t-k+2\leq t+1,
\]
which further implies
\begin{equation}\label{eq:nes-tv-term1}
a_{t,1}+\sum_{k=1}^th_{t,k}\leq t(t+1)+\frac{t(t+3)}{2}\leq4t^2.
\end{equation}
Moreover, note that
\begin{align*}
    a_{t,k+1}\leq h_{t,k+1}(h_{t,k+1}-1)\leq (t-k+1)(t-k)\leq t^2,~\forall k\ge 1,
\end{align*}
which together with $h_{t,k}^2\leq(t+1)^2\leq4t^2$ gives
\begin{equation}\label{eq:nes-tv-term2}
(1+\alpha^2\gamma^2)a_{t,k+1}-(1-\alpha\gamma)h_{t,k}^2\leq (1+\alpha^2\gamma^2) t^2 +\alpha\gamma h_{t,k}^2\leq (1+4\alpha\gamma+\alpha^2\gamma^2)t^2.
\end{equation}
Plugging the above estimates into Eq.~\eqref{eq:cvx-stab-nesterov-tv-convex} and using $\ebb[L_S(\bw_1)]\leq \sum_{k=1}^{t} \ebb[L_S(\bw_k)]$ gives Eq.~\eqref{eq:nesterov-tv-convex-polynomial}. The proof is complete.
\end{proof}

\begin{remark}[Proof Comparison]\label{rm:nesterov-tv-proof-route}\normalfont
Recall that the proof of Theorem~\ref{thm:stab-cvx} directly estimates $\|\bw_{t+1}-\bw_{t+1}^{(i)}\|^2$ and uses co-coercivity to absorb the positive gradient-difference terms generated by momentum. However, using similar arguments for~\eqref{alg:sgd-nesterov-tv} will lead to a restrictive step size depending on $t$. To see this, recall in the proof of Theorem~\ref{thm:stab-cvx} that a key step is to make the coefficient of $\ebb[\|\bg_k-\bg_k^{(i)}\|^2\mid i_k\neq i]$ in Eq.~\eqref{eq:cvx-stab-9} nonpositive, i.e.,
\[
A_{t,k}-B_{t,k}/\alpha\leq0,\qquad \forall\,1\leq k\leq t-1.
\]
Now let us consider the analogue of this step for~\eqref{alg:sgd-nesterov-tv}. In view of Eqs.~\eqref{eq:sgd-nes-1}-\eqref{eq:sgd-nes-2}, unrolling the momentum as in Lemma~\ref{lem:ip-w-m} then gives the coefficient $\prod_{r=k}^{t-1}\beta_r$ for $\langle\bw_k-\bw_k^{(i)},\bg_k-\bg_k^{(i)}\rangle$ in Eq.~\eqref{eq:cvx-stab-5}, which is then scaled by the step size $\eta_t:=\gamma\beta_t$ from Eq.~\eqref{eq:cvx-stab-3}. Hence $B_{t,k}=\eta_t\prod_{r=k}^{t-1}\beta_r=\gamma\prod_{r=k}^{t}\beta_r$ instead of the constant momentum case $\eta \beta^{t-k}$. Similar computation for $A_{t,k}$ then implies that letting $A_{t,k}-B_{t,k}/\alpha\leq0$ forces the step size condition Eq.~\eqref{eq:stab-lr} to hold with $\beta$ replaced by all $\beta_r$ with $k\leq r\leq t$. Denote $\bar\beta_{t}:=\max_{1\leq r\leq t}\beta_r$, then this amounts to requiring
\[
\frac{(1+\bar\beta_{t})(3-\bar\beta_{t})}{(1-\bar\beta_{t})^2}\eta_t+\frac{\bar\beta_{t}^2+3}{2(1-\bar\beta_{t})^2}\gamma\leq\frac{1}{\alpha}.
\]
Since $\eta_t=\gamma\beta_t\leq\gamma\bar\beta_{t}$, the above condition in turn implies
\begin{equation}\label{eq:gamma-cond-tv-fake}
\gamma\leq\frac{2(1-\bar\beta_{t})^2}{\alpha(3+6\bar\beta_{t}+5\bar\beta_{t}^2-2\bar\beta_{t}^3)}\asymp\frac{(1-\bar\beta_{t})^2}{\alpha},
\end{equation}
which is exactly the condition in Corollary~\ref{thm:stab-cvx-nesterov} with $\beta$ replaced by $\bar\beta_{t}$.
For the classic Nesterov momentum schedule $\beta_k=k/(k+3)$, we have $\bar{\beta}_{t}=1-\Theta(1/t)$, so Eq.~\eqref{eq:gamma-cond-tv-fake} only permits $\gamma\lesssim 1/(\alpha t^2)$.

To avoid this restriction, we establish Proposition~\ref{prop:stab-cvx-nesterov-tv-convex} through a different strategy. We first focus on the optimization trajectory to obtain $\ebb[\|\bw_{t+1}-\bw_1\|^2]$. Then we connect it to the stability bound through an indicator random variable $J_{t,i}$ in Eq.~\eqref{eq:w-wi-1}. Finally, from the symmetry between $\bz_i$ and $\bz_i'$ and the observation $\sum_{i=1}^{n}J_{t,i}\leq \min\{t,n\}$, we obtain the desired stability bound by averaging over $i\in[n]$.
\end{remark}
}

\subsection{Proofs of Optimization Error Bounds}\label{sec:proof-opt}
In this subsection, we provide omitted proofs for the results established in Section \ref{sec:opt}. We first prove the two auxiliary lemmas for the introduced sequence $\{\by_k\}$ defined in Eq.~\eqref{eq:y-def}.

\medskip\noindent\textbf{Lemma~\ref{lem:y-gd}} (Restated). \textit{Consider the sequence $\{\by_k\}_{k\ge 1}$ defined in Eq.~\eqref{eq:y-def}, where $\{\bw_k\}_{k\ge 1}$ is generated by the generalized SGDM Eq.~\eqref{eq:m-update}-\eqref{eq:w-update}. Then for any $k\ge 1$ and $\beta\in [0,1)$, it holds that
\[
\by_{k+1}=\by_k-\Big(\gamma+\frac{\eta}{1-\beta}\Big) \bg_k.
\]}
\medskip

\begin{proof}[of Lemma \ref{lem:y-gd}]
It suffices to verify that
\[
(1-\beta)(\by_{k+1}-\by_k)=-[(1-\beta)\gamma+\eta] \bg_k,~\forall k\ge 1.
\]
Let us first consider the case $k=1$. Note that $\bfm_1=\bg_1$ since $\bfm_0=\mathbf{0}$, hence we know that $\bw_2 - \bw_1 = - \gamma\bg_1-\eta \bg_1$ from Eq.~\eqref{eq:w-update}, which further implies that
\begin{align*}
(1-\beta)(\by_2-\by_1) & = \bw_2-\beta \bw_1+\beta \gamma \bg_1-(1-\beta) \bw_1=\bw_2-\bw_1+\beta \gamma \bg_1 \\
& =-\gamma \bg_1-\eta \bg_1+\beta \gamma \bg_1=-[(1-\beta)\gamma+\eta] \bg_1.
\end{align*}
For $k\ge 2$, it follows from Eq.~\eqref{eq:y-def}, together with the update rules Eq.~\eqref{eq:m-update}-\eqref{eq:w-update} that
\begin{align*}
(1-\beta)(\by_{k+1}-\by_k) & =\bw_{k+1}-\beta \bw_k+\beta \gamma \bg_k-(\bw_k-\beta \bw_{k-1}+\beta \gamma \bg_{k-1}) \\
& =\bw_{k+1}-\bw_k-\beta(\bw_k-\bw_{k-1})+\beta \gamma \bg_k-\beta \gamma \bg_{k-1} \\
& =-\gamma \bg_k-\eta \bfm_k+\beta \gamma \bg_{k-1}+\beta \eta \bfm_{k-1}+\beta \gamma \bg_k-\beta \gamma \bg_{k-1} \\
& =-(1-\beta) \gamma \bg_k-\eta(\bfm_k-\beta \bfm_{k-1}) \\
& =-(1-\beta) \gamma \bg_k-\eta \bg_k =-[(1-\beta) \gamma+\eta] \bg_k.
\end{align*}
The proof is now complete.
\end{proof}

\medskip\noindent\textbf{Lemma~\ref{lem:dist-yw}} (Restated). \textit{Consider the sequence $\{\by_k\}_{k\ge 1}$ defined in Eq.~\eqref{eq:y-def}, where $\{\bw_k\}_{k\ge 1}$ is generated by the generalized SGDM Eq.~\eqref{eq:m-update}-\eqref{eq:w-update}. Then for any $k\ge 1$ and $\beta\in[0,1)$, it holds that
\[
\|\by_k-\bw_k\|^2 \leq \frac{\beta^2}{(1-\beta)^3} \eta^2 \sum_{j=1}^{k-1} \beta^{k-1-j}\|\bg_j\|^2.
\]}
\medskip

\begin{proof}[of Lemma \ref{lem:dist-yw}]
    If $\beta=0$, then the result holds trivially since $\by_k=\bw_k$ for all $k\ge 1$ according to Eq.~\eqref{eq:y-def}. Hence we assume $\beta\in(0,1)$ in the following. We prove the result by showing that
    \begin{equation}\label{eq:dist-yw-1}
        \|\by_k-\bw_k\|^2 = \frac{\beta^2}{(1-\beta)^2} \eta^2 \|\bfm_{k-1}\|^2,~\forall k\ge 1.
    \end{equation}
    Suppose Eq.~\eqref{eq:dist-yw-1} holds for now. Note from the update rule Eq.~\eqref{eq:m-update} that $\bfm_k=\sum_{j=1}^k \beta^{k-j} \bg_j$, then it follows from Eq.~\eqref{eq:dist-yw-1} and the convexity of $\|\cdot\|^2$ that
    \begin{align*}
        \|\by_k-\bw_k\|^2 = & \frac{\beta^2}{(1-\beta)^2} \eta^2 \Big\|\sum_{j=1}^{k-1} \beta^{k-1-j} \bg_j\Big\|^2 = \frac{\beta^2}{(1-\beta)^2} \eta^2 \Big\|\Big(\sum_{j=1}^{k-1} \beta^{k-1-j}\Big)\sum_{j=1}^{k-1} \frac{\beta^{k-1-j}}{\sum_{j=1}^{k-1} \beta^{k-1-j}} \bg_j\Big\|^2\\
        \leq & \frac{\beta^2}{(1-\beta)^2} \eta^2\Big(\sum_{j=1}^{k-1} \beta^{k-1-j}\Big)^2 \sum_{j=1}^{k-1}\frac{\beta^{k-1-j}}{\sum_{j=1}^{k-1} \beta^{k-1-j}} \|\bg_j\|^2 \\
        = & \frac{\beta^2}{(1-\beta)^2} \eta^2\Big(\sum_{j=1}^{k-1} \beta^{k-1-j}\Big) \sum_{j=1}^{k-1}\beta^{k-1-j} \|\bg_j\|^2 \leq \frac{\beta^2}{(1-\beta)^3} \eta^2 \sum_{j=1}^{k-1} \beta^{k-1-j}\| \bg_j\|^2,
    \end{align*}
    where we apply $\sum_{j=1}^{k-1} \beta^{k-1-j}=\sum_{j=0}^{k-2} \beta^j=\frac{1-\beta^{k-1}}{1-\beta}\leq \frac{1}{1-\beta}$ to obtain the last inequality.

    Hence it remains to prove Eq.~\eqref{eq:dist-yw-1}. Since $\bfm_0=\mathbf{0}$ and $\by_1= \bw_1$, Eq.~\eqref{eq:dist-yw-1} holds for $k=1$. For $k\ge 2$, we have from the definition of $\by_k$ and the update rule Eq.~\eqref{eq:w-update} that
    \begin{align*}
        \|\by_k-\bw_k\|^2 = &\Big\|\frac{1}{1-\beta} \bw_k-\frac{\beta}{1-\beta} \bw_{k-1}+\frac{\beta \gamma}{1-\beta} \bg_{k-1}-\bw_k\Big\|^2\\
        = & \frac{\beta^2}{(1-\beta)^2}\|\bw_k-\bw_{k-1}+\gamma\bg_{k-1}\|^2\\
        = & \frac{\beta^2}{(1-\beta)^2}\eta^2\|\bfm_{k-1}\|^2,
    \end{align*}
    which completes the proof.
\end{proof}
Now we are ready to present the proof of Theorem \ref{thm:cvx-opt}, which provides optimization error bounds for the generalized SGDM.

\medskip\noindent\textbf{Theorem~\ref{thm:cvx-opt}} (Restated). \textit{Assume for any $\bz\in\zcal$, the map $\bw\mapsto\ell(\bw;\bz)$ is nonnegative, convex, and $\alpha$-smooth.
Let $\{\bw_t\}_{t\ge 1}$ be the sequence produced by the generalized SGDM Eq.~\eqref{eq:m-update}-\eqref{eq:w-update} with $\beta\in [0, 1)$ applied to $S$ and denote  $\bar{\bw}_t:=\frac{1}{t}\sum_{k=1}^t \bw_k$. If the step sizes $\gamma\ge 0$ and $\eta>0$ are chosen such that
\[
    \gamma^2+\frac{2}{1-\beta}\gamma\eta+\frac{1+\beta}{(1-\beta)^2}\eta^2\leq \frac{2(1-\beta)\gamma+2\eta}{3\alpha},
\]
then for any $\bw\in\wcal$ and $t\ge 1$, we have
\begin{align*}
   \ebb_A[L_S(\bar{\bw}_t)]-L_S(\bw)
   \leq\frac{3(1-\beta)\|\bw_1-\bw\|^2}{2\big((1-\beta)\gamma+\eta\big) t}+ \Big((1-\beta)\gamma+\eta +\frac{\beta\eta^2}{(1-\beta)\gamma+\eta}\Big)\frac{3\alpha L_S(\bw)}{(1-\beta)^2}.
\end{align*}}
\medskip

\begin{proof}[of Theorem \ref{thm:cvx-opt}]
For notational convenience, denote the effective step size $\tilde{\gamma}:=\gamma+\frac{\eta}{1-\beta}$, then it follows from Lemma \ref{lem:y-gd} that
\[
\by_{k+1}=\by_k-\tilde{\gamma} \bg_k,~\forall k\ge 1.
\]
Using the above inequality, we can bound $\|\by_{k+1}-\bw\|^2$ for any $\bw\in \wcal$ as follows:
\begin{align}
&\|\by_{k+1}-\bw\|^2 = \|\by_k-\bw\|^2+\|\by_{k+1}-\by_k\|^2+2\langle \by_{k+1}-\by_k, \by_k-\bw\rangle \notag\\
= & \|\by_k-\bw\|^2+\tilde{\gamma}^2\|\bg_k\|^2-2 \tilde{\gamma}\langle \bg_k, \by_k-\bw_k\rangle+2 \tilde{\gamma}\langle \bg_k, \bw-\bw_k\rangle \notag\\
\leq & \|\by_k-\bw\|^2+\tilde{\gamma}^2\|\bg_k\|^2+\frac{1-\beta}{\beta}\|\by_k-\bw_k\|^2+\frac{\beta}{1-\beta}\tilde{\gamma}^2\|\bg_k\|^2+2 \tilde{\gamma}\langle \bg_k, \bw-\bw_k\rangle \notag\\
\leq & \|\by_k-\bw\|^2+\frac{1}{1-\beta}\tilde{\gamma}^2\|\bg_k\|^2+\frac{\beta}{(1-\beta)^2} \eta^2 \sum_{j=1}^{k-1} \beta^{k-1-j}\|\bg_j\|^2+2 \tilde{\gamma}\langle \bg_k, \bw-\bw_k\rangle \label{eq:cvx-opt-m1},
\end{align}
where we apply Young's inequality to obtain the first inequality, and the last inequality follows from Lemma \ref{lem:dist-yw}. Note that the resulting inequality in Eq.~\eqref{eq:cvx-opt-m1} holds for $\beta=0$ since $\by_k=\bw_k$ in this case.

To proceed, note from the unbiasedness of $\bg_k$ and convexity of $L_S$ that
\begin{equation}\label{eq:cvx-opt-1}
    \ebb_A[\langle \bg_k, \bw-\bw_k\rangle]=\ebb_A[\langle\nabla L_S(\bw_k), \bw-\bw_k\rangle] \leq \ebb_A[L_S(\bw)-L_S(\bw_k)].
\end{equation}
Additionally, from the self-bounding property of $\ell$ (Lemma~\ref{lem:self-bounding}), we obtain
\begin{equation}\label{eq:cvx-opt-2}
    \ebb_A[\|\bg_k\|^2]\leq 2\alpha \ebb_A[\ell(\bw_k;\bz_{i_k})]=2\alpha \ebb_A[L_S(\bw_k)].
\end{equation}
Taking expectations $\ebb_A[\cdot]$ on both sides of Eq.~\eqref{eq:cvx-opt-m1}, and applying Eq.~\eqref{eq:cvx-opt-1} and Eq.~\eqref{eq:cvx-opt-2} gives
\begin{multline*}
\ebb_A[\|\by_{k+1}-\bw\|^2-\|\by_k-\bw\|^2]\leq \frac{2\alpha}{1-\beta}\tilde{\gamma}^2 \ebb_A[L_S(\bw_k)]\\
+\frac{2\alpha\beta}{(1-\beta)^2} \eta^2 \sum_{j=1}^{k-1} \beta^{k-1-j} \ebb_A[L_S(\bw_j)]+2 \tilde{\gamma} \ebb_A[L_S(\bw)-L_S(\bw_k)].
\end{multline*}
Summing over the above inequality for $k=1,\dots, t$ and noting that $\by_1=\bw_1$ further leads to
\begin{multline}\label{eq:cvx-opt-m2}
    \ebb_A[\|\by_{t+1}-\bw\|^2]-\|\bw_1-\bw\|^2\leq \Big(\frac{2\alpha}{1-\beta} \tilde{\gamma}^2-2 \tilde{\gamma}\Big) \sum_{k=1}^t \ebb_A[L_S(\bw_k)]\\
    +\frac{2\alpha\beta}{(1-\beta)^2} \eta^2 \sum_{k=2}^t \sum_{j=1}^{k-1} \beta^{k-1-j} \ebb_A[L_S(\bw_j)]+2 \tilde{\gamma} t L_S(\bw).
\end{multline}
Note that
\begin{align*}
\sum_{k=2}^t \sum_{j=1}^{k-1} \beta^{k-1-j} \ebb_A[L_S(\bw_j)] = &\sum_{k=1}^{t-1} \sum_{j=1}^k \beta^{k-j} \ebb_A[L_S(\bw_j)]=\sum_{j=1}^{t-1}\Big(\sum_{k=j}^{t-1} \beta^{k-j}\Big) \ebb_A[L_S(\bw_j)] \\
\leq & \frac{1}{1-\beta} \sum_{j=1}^{t-1} \ebb_A[L_S(\bw_j)]\leq\frac{1}{1-\beta} \sum_{k=1}^t \ebb_A[L_S(\bw_k)],
\end{align*}
where we exchange the order of the two summations to obtain the second equality, and the first inequality follows from $\sum_{k=j}^{t-1} \beta^{k-j}=\sum_{k=0}^{t-j-1} \beta^k=\frac{1-\beta^{t-j}}{1-\beta}\leq \frac{1}{1-\beta}$ (note we use the convention $0^0=1$ in case $\beta=0$). Plugging the above inequality into Eq.~\eqref{eq:cvx-opt-m2} leads to
\begin{align}
&\ebb_A[\|\by_{t+1}-\bw\|^2]-\|\bw_1-\bw\|^2\notag\\
\leq & \Big(\frac{2\alpha}{1-\beta} \tilde{\gamma}^2-2 \tilde{\gamma}\Big) \sum_{k=1}^t \ebb_A[L_S(\bw_k)]+\frac{2\alpha\beta}{(1-\beta)^3} \eta^2 \sum_{k=1}^t \ebb_A[L_S(\bw_k)]+2 \tilde{\gamma} t L_S(\bw)\notag\\
= & \Big(\frac{2\alpha}{1-\beta}\tilde{\gamma}^2+\frac{2\alpha\beta}{(1-\beta)^3} \eta^2-2\tilde{\gamma}\Big) \sum_{k=1}^t \ebb_A[L_S(\bw_k)]+2 \tilde{\gamma} t L_S(\bw)\label{eq:cvx-opt-m3}.
\end{align}
Then it follows from Eq.~\eqref{eq:cvx-opt-m3} that
\begin{multline}\label{eq-1121}
-\|\bw_1-\bw\|^2\leq\ebb_A[\|\by_{t+1}-\bw\|^2]-\|\bw_1-\bw\|^2 \\
\leq\Big(\frac{2\alpha}{1-\beta}\tilde{\gamma}^2+\frac{2\alpha\beta}{(1-\beta)^3} \eta^2-2\tilde{\gamma}\Big) \sum_{k=1}^t \ebb_A[L_S(\bw_k)-L_S(\bw)]+ \Big(\frac{2\alpha}{1-\beta}\tilde{\gamma}^2+\frac{2\alpha\beta}{(1-\beta)^3} \eta^2\Big) t L_S(\bw).
\end{multline}

We consider two cases for the term $\sum_{k=1}^t \ebb_A[L_S(\bw_k)-L_S(\bw)]$. Firstly, if $\sum_{k=1}^t \ebb_A[L_S(\bw_k)-L_S(\bw)]\leq 0$, then it follows from the convexity of $L_S$ that
\[
\ebb_A[L_S(\bar{\bw}_t)]\leq \frac{1}{t}\sum_{k=1}^t \ebb_A[L_S(\bw_k)]\leq \ebb_A[L_S(\bw)],
\]
which leads to the stated bound in the theorem directly.

We now assume $\sum_{k=1}^t \ebb_A[L_S(\bw_k)-L_S(\bw)]> 0$. From the step size condition Eq.~\eqref{eq:opt-lr} and $\tilde{\gamma}=\gamma+\eta/(1-\beta)$, we know that
\begin{align}
\frac{2\alpha}{1-\beta}\tilde{\gamma}^2+\frac{2\alpha\beta}{(1-\beta)^3} \eta^2 &= \frac{2\alpha}{1-\beta} \gamma^2+\frac{4\alpha}{(1-\beta)^2}\gamma \eta+\frac{2\alpha}{(1-\beta)^3}\eta^2 + \frac{2\alpha\beta}{(1-\beta)^3} \eta^2 \notag\\
&= \frac{2\alpha}{1-\beta} \gamma^2+\frac{4\alpha}{(1-\beta)^2}\gamma \eta+\frac{2\alpha(1+\beta)}{(1-\beta)^3}\eta^2 \notag\\
& = \frac{2\alpha}{1-\beta}\Big(\gamma^2+\frac{2}{1-\beta}\gamma\eta+\frac{1+\beta}{(1-\beta)^2}\eta^2\Big)\notag\\
&\leq \frac{4(1-\beta)\gamma+4\eta}{3(1-\beta)} = \frac{4\tilde{\gamma}}{3}\label{eq:cvx-opt-m4}.
\end{align}
Hence, applying Eq.~\eqref{eq:cvx-opt-m4} on Eq.~\eqref{eq-1121} further gives
\begin{equation}
-\|\bw_1-\bw\|^2\leq  -\frac{2\tilde{\gamma}}{3} \sum_{k=1}^t \ebb_A[L_S(\bw_k)-L_S(\bw)]+\Big(\frac{2\alpha}{1-\beta}\tilde{\gamma}^2+\frac{2\alpha\beta}{(1-\beta)^3} \eta^2\Big) t L_S(\bw)\label{eq:cvx-opt-m5}.
\end{equation}
Rearranging the terms in Eq.~\eqref{eq:cvx-opt-m5} and dividing both sides by $2\tilde{\gamma} t/3$, we obtain
\begin{align}
    &\frac{1}{t}\sum_{k=1}^t \ebb_A[L_S(\bw_k)]-L_S(\bw)
    \leq \frac{3\|\bw_1-\bw\|^2}{2\tilde{\gamma}t}+ \Big(\frac{3\alpha}{1-\beta}\tilde{\gamma} +\frac{3\alpha\beta\eta^2}{(1-\beta)^3\tilde{\gamma}}\Big) L_S(\bw)\notag\\
    =& \frac{3(1-\beta)\|\bw_1-\bw\|^2}{2\big((1-\beta)\gamma+\eta\big) t}+ \Big(\frac{3\alpha\big((1-\beta)\gamma+\eta\big)}{(1-\beta)^2} +\frac{3\alpha\beta\eta^2}{(1-\beta)^2\big((1-\beta)\gamma+\eta\big)}\Big) L_S(\bw),\label{eq:opt-bound}
\end{align}
where we plug in $\tilde{\gamma}=\gamma+\eta/(1-\beta)$ to get the equality. Then the proof is complete from the convexity of $L_S$, i.e., $L_S(\bar{\bw}_t)\leq\frac{1}{t}\sum_{k=1}^t L_S(\bw_k)$.
\end{proof}

\subsection{Proofs of Excess Population Risk Bounds}\label{sec:proof-epr}
In this subsection, we present the omitted proofs for the results established in Section \ref{sec:epr} regarding excess population risk (EPR) bounds. We start with the proof of Theorem \ref{thm:epr-cvx}, which provides the EPR bounds for the generalized SGDM.

\medskip\noindent\textbf{Theorem~\ref{thm:epr-cvx}} (Restated). \textit{Assume that for any $\bz\in\zcal$, the map $\bw\mapsto\ell(\bw;\bz)$ is nonnegative, convex, and $\alpha$-smooth.
    Let $\{\bw_t\}_{t\ge 1}$ be the sequence produced by the generalized SGDM Eq.~\eqref{eq:m-update}-\eqref{eq:w-update} with $\beta \in[0,1)$ applied to $S$ and denote $\bar{\bw}_t:=\frac{1}{t}\sum_{k=1}^t \bw_k$. If the step sizes $\gamma\ge 0$ and $\eta>0$ are chosen such that
\begin{align*}
    & \frac{(1+\beta)(3-\beta)}{(1-\beta)^2}\eta+\frac{\beta^2+3}{2(1-\beta)^2}\gamma\leq \frac{1}{\alpha},\\
    & \gamma^2+\frac{2}{1-\beta}\gamma\eta+\frac{1+\beta}{(1-\beta)^2}\eta^2\leq \frac{2(1-\beta)\gamma+2\eta}{3\alpha},
\end{align*}
    then for any $t\ge 1$ and $\rho>0$, we have
\begin{align*}
& \e[L(\bar{\bw}_t)]-L(\bw^*) \\
\lesssim & \frac{1-\beta}{\big((1-\beta)\gamma+\eta\big) t}+ D(\beta, \gamma, \eta) L(\bw^*) + \frac{1}{\rho}\Big(L(\bw^*)+\frac{1-\beta}{\big((1-\beta)\gamma+\eta\big) t}+ D(\beta, \gamma, \eta) L(\bw^*)\Big) \\
&+ \frac{(1+\rho) C(\beta, \gamma, \eta, t)}{n} \Big(t L(\bw^*)+\frac{1-\beta}{(1-\beta)\gamma+\eta}+ t D(\beta, \gamma, \eta) L(\bw^*)\Big),
\end{align*}
where we define
\begin{align*}
    & C(\beta, \gamma, \eta, t):=\frac{\eta^2}{(1-\beta)^3}+\gamma^2+\frac{\beta \gamma\eta}{(1-\beta)^3}+\frac{(\gamma+\eta)((1-\beta)\gamma+\eta)t}{(1-\beta)^2 n},\\
    & D(\beta, \gamma, \eta):=\frac{(1-\beta)\gamma+\eta}{(1-\beta)^2} +\frac{\beta\eta^2}{(1-\beta)^2\big((1-\beta)\gamma+\eta\big)}.
\end{align*}}
\medskip

\begin{proof}[of Theorem~\ref{thm:epr-cvx}]
It follows from Theorem \ref{thm:stab-cvx} and the definition of $C(\beta, \gamma, \eta, t)$ that (the dependency on $\alpha$ is absorbed in the notation $\lesssim$)
\begin{align}
    &\ebb[\|\bw_{t+1}-\bw_{t+1}^{(i)}\|^2] \notag \\
\leq &\Big(\frac{2(1+\beta)^2}{(1-\beta)^3}\eta^2+2\gamma^2+\frac{\beta(\beta^2+3)}{(1-\beta)^3}\gamma\eta+\frac{(\gamma+\eta)((1-\beta)\gamma+\eta)t}{(1-\beta)^2 n}\Big)\frac{8 \alpha e}{n} \sum_{k=1}^t \ebb[L_S(\bw_k)]\notag\\
\lesssim &\frac{C(\beta, \gamma, \eta, t)}{n} \sum_{k=1}^t \ebb[L_S(\bw_k)].\label{eq:epr-cvx1}
\end{align}
Note that the RHS of Eq.~\eqref{eq:epr-cvx1} is an increasing function in terms of $t$. Hence, we know from the convexity of $\|\cdot\|^2$ that
\begin{equation}\label{eq:epr-cvx2}
    \e[\|\bar{\bw}_t-\bar{\bw}_t^{(i)}\|^2]\leq\frac{1}{t}\sum_{k=1}^t\e[\|\bw_k-\bw_k^{(i)}\|^2] \lesssim \frac{C(\beta, \gamma, \eta, t)}{n}\sum_{k=1}^t \ebb[L_S(\bw_k)].
\end{equation}
Applying Lemma \ref{lem:general-stablity} and plugging in Eq.~\eqref{eq:epr-cvx2} leads to
\begin{align}
    \e[L(\bar{\bw}_t)-L_S(\bar{\bw}_t)] \leq &\frac{\alpha}{\rho} \e[L_S(\bar{\bw}_t)]+\frac{\alpha+\rho}{2 n} \sum_{i=1}^n \e[\|\bar{\bw}_t-\bar{\bw}_t^{(i)}\|^2]\notag\\
    \lesssim &\frac{1}{\rho} \e[L_S(\bar{\bw}_t)]+\frac{(1+\rho) C(\beta, \gamma, \eta, t)}{n} \sum_{k=1}^t \ebb[L_S(\bw_k)]\label{eq:epr-cvx3}.
\end{align}
Next we bound the term involving $\e[L_S(\bar{\bw}_t)]$ and $\sum_{k=1}^t \e[L_S(\bw_k)]$.
To this end, applying Eq.~\eqref{eq:opt-bound} with $\bw=\bw^*$ from the proof of Theorem \ref{thm:cvx-opt} leads to (the dependency on $\|\bw_1-\bw^*\|_2$ is absorbed in the notation $\lesssim$)
\begin{align}
\e[L_S(\bar{\bw}_t)] \leq & L(\bw^*)+\frac{3(1-\beta)\|\bw_1-\bw^*\|^2}{2\big((1-\beta)\gamma+\eta\big) t}+ \Big((1-\beta)\gamma+\eta +\frac{\beta\eta^2}{(1-\beta)\gamma+\eta}\Big)\frac{3\alpha L(\bw^*)}{(1-\beta)^2}\notag\\
\lesssim & L(\bw^*)+\frac{1-\beta}{\big((1-\beta)\gamma+\eta\big) t} + D(\beta, \gamma, \eta) L(\bw^*) \label{eq:epr-cvx4}
\end{align}
and similarly
\begin{equation}\label{eq:epr-cvx5}
\sum_{k=1}^t \e[L_S(\bw_k)] \lesssim t L(\bw^*)+\frac{1-\beta}{(1-\beta)\gamma+\eta}+ t D(\beta, \gamma, \eta) L(\bw^*).
\end{equation}
Plugging Eq.~\eqref{eq:epr-cvx4} and Eq.~\eqref{eq:epr-cvx5} into Eq.~\eqref{eq:epr-cvx3} gives
\begin{multline}\label{eq:epr-cvx-gen}
   \e[L(\bar{\bw}_t)-L_S(\bar{\bw}_t)] \lesssim \frac{1}{\rho}\Big(L(\bw^*)+\frac{1-\beta}{\big((1-\beta)\gamma+\eta\big) t}+ D(\beta, \gamma, \eta) L(\bw^*)\Big) \\+ \frac{(1+\rho) C(\beta, \gamma, \eta, t)}{n} \Big(t L(\bw^*)+\frac{1-\beta}{(1-\beta)\gamma+\eta}+ t D(\beta, \gamma, \eta) L(\bw^*)\Big).
\end{multline}
Combining the generalization error in Eq.~\eqref{eq:epr-cvx-gen} and optimization error in Eq.~\eqref{eq:epr-cvx4} leads to
\begin{align*}
& \e[L(\bar{\bw}_t)]-L(\bw^*) = \e[L(\bar{\bw}_t)-L_S(\bar{\bw}_t)] + \e[L_S(\bar{\bw}_t)-L(\bw^*)]\notag\\
\lesssim & \frac{1-\beta}{\big((1-\beta)\gamma+\eta\big) t}+ D(\beta, \gamma, \eta) L(\bw^*) + \frac{1}{\rho}\Big(L(\bw^*)+\frac{1-\beta}{\big((1-\beta)\gamma+\eta\big) t}+ D(\beta, \gamma, \eta) L(\bw^*)\Big) \\
&+ \frac{(1+\rho) C(\beta, \gamma, \eta, t)}{n} \Big(t L(\bw^*)+\frac{1-\beta}{(1-\beta)\gamma+\eta}+ t D(\beta, \gamma, \eta) L(\bw^*)\Big),
\end{align*}
which completes the proof.
\end{proof}
We proceed to the proof of Theorem \ref{thm:epr-polyak}, which establishes general EPR bounds for SGD with Polyak's momentum \eqref{alg:sgd-polyak}. The proof builds heavily on Theorem \ref{thm:epr-cvx}.

\medskip\noindent\textbf{Theorem~\ref{thm:epr-polyak}} (Restated). \textit{Assume that for any $\bz\in\zcal$, the map $\bw\mapsto\ell(\bw;\bz)$ is nonnegative, convex, and $\alpha$-smooth.
    Let $\{\bw_t\}_{t\ge 1}$ be the sequence produced by SGD with Polyak's momentum \eqref{alg:sgd-polyak} with $\beta\in [0,1)$ applied to $S$ and denote $\bar{\bw}_t:=\frac{1}{t}\sum_{k=1}^t \bw_k$. If the step size $\eta>0$ is chosen such that
    \[\eta\leq\frac{(1-\beta)^2}{\alpha(1+\beta)(3-\beta)},\]
    then for any $t\ge 1$ and $\rho>0$, we have
\begin{multline*}
\e[L(\bar{\bw}_t)]-L(\bw^*) \lesssim \frac{1-\beta}{\eta t}+ \frac{\eta}{(1-\beta)^2} L(\bw^*) + \frac{1}{\rho}\Big(L(\bw^*)+\frac{1-\beta}{\eta t}+ \frac{\eta}{(1-\beta)^2} L(\bw^*)\Big) \\
+ \frac{1+\rho}{n} \frac{\eta^2}{(1-\beta)^2}\Big(\frac{1}{1-\beta}+\frac{t}{n}\Big) \Big(t L(\bw^*)+\frac{1-\beta}{\eta}+ \frac{\eta t}{(1-\beta)^2}  L(\bw^*)\Big).
\end{multline*}}
\medskip

\begin{proof}[of Theorem \ref{thm:epr-polyak}]
    Setting $\gamma=0$ in the definitions of $C(\beta, \gamma, \eta, t)$ and $D(\beta, \gamma, \eta)$ gives
    \[
    C(\beta, 0, \eta, t):=\frac{\eta^2}{(1-\beta)^2}\Big(\frac{1}{1-\beta}+\frac{t}{n}\Big),\quad D(\beta, 0, \eta):=\frac{(1+\beta)\eta}{(1-\beta)^2}.
    \]
    Plugging the above two quantities into the EPR bound established in Theorem \ref{thm:epr-cvx} leads to
    \begin{multline*}
\e[L(\bar{\bw}_t)]-L(\bw^*) \lesssim \frac{1-\beta}{\eta t}+ \frac{(1+\beta)\eta}{(1-\beta)^2} L(\bw^*) + \frac{1}{\rho}\Big(L(\bw^*)+\frac{1-\beta}{\eta t}+ \frac{(1+\beta)\eta}{(1-\beta)^2} L(\bw^*)\Big) \\
+ \frac{1+\rho}{n} \frac{\eta^2}{(1-\beta)^2}\Big(\frac{1}{1-\beta}+\frac{t}{n}\Big) \Big(t L(\bw^*)+\frac{1-\beta}{\eta}+ \frac{(1+\beta)\eta}{(1-\beta)^2} t  L(\bw^*)\Big),
\end{multline*}
which leads to the desired result after omitting the constant $1+\beta$. Additionally, setting $\gamma=0$ in Eq.~\eqref{eq:epr-lr1} gives
\[
\frac{(1+\beta)(3-\beta)}{(1-\beta)^2}\eta \leq \frac{1}{\alpha},
\]
which is equivalent to the given step size condition $\eta\leq\frac{(1-\beta)^2}{\alpha(1+\beta)(3-\beta)}$. Hence, it remains to verify Eq.~\eqref{eq:epr-lr2}. To this end, note from $0\leq\beta<1$ that
\[
\frac{1+\beta}{(1-\beta)^2}\eta^2\leq \frac{1+\beta}{(1-\beta)^2}\frac{(1-\beta)^2}{\alpha(1+\beta)(3-\beta)}\eta=\frac{\eta}{\alpha(3-\beta)}\leq \frac{\eta}{2\alpha}\leq\frac{2\eta}{3\alpha},
\]
which corresponds to Eq.~\eqref{eq:epr-lr2} with $\gamma=0$. The proof is complete.
\end{proof}
Next, we prove Corollary \ref{co:epr-polyak}, which establishes optimal EPR bounds for \eqref{alg:sgd-polyak} with appropriate choices of parameters.

\medskip\noindent\textbf{Corollary~\ref{co:epr-polyak}} (Restated). \textit{Consider the same setting as in Theorem \ref{thm:epr-polyak} and $\beta\in [0,1)$.
For any $t\gtrsim n/(1-\beta)$, the following holds:
\begin{itemize}
    \item[1.] If $L(\bw^*) \ge \frac{1}{n}$, choose $\rho = \sqrt{n L(\bw^*)}$ and
    $\eta\asymp\frac{(1-\beta)n}{t\sqrt{n L(\bw^*)}}$ to obtain
    \[
    \ebb[L(\bar{\bw}_t)-L(\bw^*)]\lesssim \sqrt{\frac{L(\bw^*)}{n}}.
    \]
    \item[2.] If $L(\bw^*) < \frac{1}{n}$, choose $\rho = 1$ and
    $\eta\asymp\frac{(1-\beta)n}{t}$ to obtain
    \[
    \ebb[L(\bar{\bw}_t)-L(\bw^*)]\lesssim \frac{1}{n}.
    \]
\end{itemize}}
\medskip

\begin{proof}[of Corollary \ref{co:epr-polyak}]
\wzm{
The condition $t\gtrsim n/(1-\beta)$ implies
\[
\frac{1}{1-\beta}+\frac{t}{n}\lesssim\frac{t}{n}.
\]
First consider $L(\bw^*)\geq 1/n$. With $\eta\asymp\frac{(1-\beta)n}{t\sqrt{nL(\bw^*)}}\lesssim \frac{(1-\beta)^2}{\sqrt{nL(\bw^*)}}$, we have
\[
\frac{\eta}{(1-\beta)^2}L(\bw^*)\lesssim\frac{L(\bw^*)}{(1-\beta)^2}\frac{(1-\beta)^2}{\sqrt{nL(\bw^*)}}\asymp\sqrt{\frac{L(\bw^*)}{n}},
\]
which further implies the following estimates:
\begin{alignat*}{2}
&\frac{1-\beta}{\eta t}\asymp\sqrt{\frac{L(\bw^*)}{n}},
&\quad&L(\bw^*)+\frac{1-\beta}{\eta t}+\frac{\eta}{(1-\beta)^2}L(\bw^*)\lesssim L(\bw^*),\\
&\frac{\eta^2}{(1-\beta)^2}\Big(\frac{1}{1-\beta}+\frac{t}{n}\Big)\lesssim\frac{1}{tL(\bw^*)},
&\quad&tL(\bw^*)+\frac{1-\beta}{\eta}+\frac{\eta t}{(1-\beta)^2}L(\bw^*)\lesssim tL(\bw^*).
\end{alignat*}
Therefore, the last term in Theorem~\ref{thm:epr-polyak} is bounded by
\[
\frac{1+\rho}{n}\cdot\frac{1}{tL(\bw^*)}\cdot tL(\bw^*) \lesssim\frac{\sqrt{nL(\bw^*)}}{n}\asymp\sqrt{\frac{L(\bw^*)}{n}},
\]
and the remaining terms have the same order. This proves the first statement.

Next suppose that $L(\bw^*)<1/n$. With $\rho=1$ and
$\eta\asymp(1-\beta)n/t$, we obtain
\begin{alignat*}{2}
&\frac{1-\beta}{\eta t}\asymp\frac{1}{n},
&\quad&\frac{\eta}{(1-\beta)^2}L(\bw^*)\lesssim L(\bw^*)<\frac{1}{n},\\
&L(\bw^*)+\frac{1-\beta}{\eta t}+\frac{\eta}{(1-\beta)^2}L(\bw^*)\lesssim\frac{1}{n},
&\quad&\frac{\eta^2}{(1-\beta)^2}\Big(\frac{1}{1-\beta}+\frac{t}{n}\Big)\lesssim\frac{n}{t},\\
&tL(\bw^*)+\frac{1-\beta}{\eta}+\frac{\eta t}{(1-\beta)^2}L(\bw^*)\lesssim\frac{t}{n}.
\end{alignat*}
Thus the last term in Theorem~\ref{thm:epr-polyak} is bounded by
\[
\frac{1+1}{n}\cdot\frac{n}{t}\cdot \frac{t}{n}\lesssim\frac{1}{n},
\]
as are all the other terms. The proof is complete.
}
\end{proof}

Finally, we present the proof of Theorem \ref{thm:epr-nesterov}, which establishes general EPR bounds for SGD with Nesterov's momentum \eqref{alg:sgd-nesterov}.

\medskip\noindent\textbf{Theorem~\ref{thm:epr-nesterov}} (Restated). \textit{Assume that for any $\bz\in\zcal$, the map $\bw\mapsto\ell(\bw;\bz)$ is nonnegative, convex, and $\alpha$-smooth.
    Let $\{\bw_t\}_{t\ge 1}$ be the sequence produced by SGD with Nesterov's momentum \eqref{alg:sgd-nesterov} with $\beta\in(0,1)$ applied to $S$ and denote $\bar{\bw}_t:=\frac{1}{t}\sum_{k=1}^t \bw_k$. Suppose the step size $\gamma>0$ is chosen such that
    \[
    \gamma\leq\frac{2(1-\beta)^2}{\alpha(3+6\beta+5\beta^2-2\beta^3)}.
    \]
    Then for any $t\ge 1$ and $\rho>0$, we have
    \begin{multline*}
\e[L(\bar{\bw}_t)]-L(\bw^*)\lesssim \frac{1-\beta}{\gamma t}+ \frac{\gamma}{(1-\beta)^2} L(\bw^*) + \frac{1}{\rho}\Big(L(\bw^*)+\frac{1-\beta}{\gamma t}+ \frac{\gamma}{(1-\beta)^2} L(\bw^*)\Big) \\
+ \frac{1+\rho}{n} \frac{\gamma^2}{(1-\beta)^2} \Big(\frac{1}{1-\beta}+\frac{t}{n}\Big) \Big(t L(\bw^*)+\frac{1-\beta}{\gamma}+ \frac{\gamma t}{(1-\beta)^2} L(\bw^*)\Big).
\end{multline*}}
\medskip

\begin{proof}[of Theorem \ref{thm:epr-nesterov}]
    Setting $\eta=\beta\gamma$ into the definition of $C(\beta, \gamma, \eta, t)$ and $D(\beta, \gamma, \eta)$ gives (note $(1-\beta)\gamma+\eta=\gamma$)
    \begin{align*}
        &C(\beta, \gamma, \beta\gamma, t):=\Big(\frac{2\beta^2+(1-\beta)^3}{(1-\beta)^3}+\frac{(1+\beta) t}{(1-\beta)^2 n}\Big)\gamma^2\lesssim \Big(\frac{1}{1-\beta}+\frac{t}{n}\Big)\frac{\gamma^2}{(1-\beta)^2},\\
        &D(\beta, \gamma, \beta\gamma):=\frac{1+\beta^3}{(1-\beta)^2}\gamma\lesssim \frac{\gamma}{(1-\beta)^2}.
    \end{align*}
    Plugging $\gamma=0$ and the above two quantities into the EPR bound established in Theorem \ref{thm:epr-cvx} leads to
    \begin{multline*}
\e[L(\bar{\bw}_t)]-L(\bw^*)\lesssim \frac{1-\beta}{\gamma t}+ \frac{\gamma}{(1-\beta)^2} L(\bw^*) + \frac{1}{\rho}\Big(L(\bw^*)+\frac{1-\beta}{\gamma t}+ \frac{\gamma}{(1-\beta)^2} L(\bw^*)\Big) \\
+ \frac{1+\rho}{n} \frac{\gamma^2}{(1-\beta)^2} \Big(\frac{1}{1-\beta}+\frac{t}{n}\Big) \Big(t L(\bw^*)+\frac{1-\beta}{\gamma}+ \frac{\gamma t}{(1-\beta)^2} L(\bw^*)\Big),
\end{multline*}
which is the desired result.
Moreover, it is straightforward to verify that $\gamma\leq\frac{2(1-\beta)^2}{\alpha(3+6\beta+5\beta^2-2\beta^3)}$ is equivalent to the step size condition Eq.~\eqref{eq:epr-lr1} with $\eta=\beta\gamma$. Hence, it remains to show that Eq.~\eqref{eq:epr-lr2} is also satisfied. To this end, setting $\eta=\beta\gamma$ shows that Eq.~\eqref{eq:epr-lr2} is equivalent to
\[
\gamma^2\Big(1+\frac{2\beta}{1-\beta}+\frac{\beta^2(1+\beta)}{(1-\beta)^2}\Big)\leq \frac{2\gamma}{3\alpha}\Longleftrightarrow\gamma\leq\frac{2(1-\beta)^2}{3\alpha(1+\beta^3)}.
\]
Then it suffices to verify that
\begin{align*}
    \frac{2(1-\beta)^2}{\alpha(3+6\beta+5\beta^2-2\beta^3)}\leq \frac{2(1-\beta)^2}{3\alpha(1+\beta^3)}&\Longleftrightarrow 3+6\beta+5\beta^2-2\beta^3\ge 3+ 3\beta^3\\
    &\Longleftrightarrow 6+5\beta-5\beta^2\ge 0,
\end{align*}
where the last inequality clearly holds due to $\beta\in[0,1)$. The proof is thus complete.
\end{proof}

\section{Experiments}\label{sec:experiments}
In this section, we conduct empirical studies to investigate the algorithmic stability of SGDM. We examine two important momentum variants: SGD with Polyak's momentum \eqref{alg:sgd-polyak} and SGD with Nesterov's momentum \eqref{alg:sgd-nesterov}, and analyze how different parameters affect their stability properties.
Specifically, we consider the binary logistic regression problem with the loss function:
\[
\ell(\bw; \bz):=\log(1+\exp(-y \langle\bw, \bx\rangle)).
\]
Here, $\bz:=(\bx, y)\in \rbb^d\times \{-1, 1\}$ represents the data sample, where $\bx\in\rbb^d$ corresponds to the feature vector and $y\in\{-1,1\}$ represents the label.
We use $8$ real-world datasets from the LIBSVM library \citep{chang2008libsvm}, whose information is summarized in Table \ref{tab:dataset}. For datasets with multiple classes, we convert them to binary classification by assigning the first half of class labels to positive ($+1$) and the second half to negative ($-1$). For each dataset, we randomly select $80\%$ of the data as the training set $S$, and we perturb a single example in $S$ to generate a neighboring dataset $S'$. We then apply the algorithms with identical parameters and initialization to both $S$ and $S'$, producing two sequences of iterates $\{\bw_t\}_{t \geq 1}$ and $\{\bw'_t\}_{t \geq 1}$. The stability is quantified by computing the Euclidean distance $d_t:= \|\bw_t - \bw'_t\|_2$ at each iteration $t$.

We design experiments to test the effect of step sizes $\gamma\ge 0$, $\eta>0$, and the momentum parameter $\beta\in[0,1)$ on stability.

\noindent\textbf{Test on effect of step sizes:}
\begin{itemize}
    \item For \eqref{alg:sgd-polyak}, fix $\beta = 0.9$ and vary $\eta \in \{0.001, 0.005, 0.025, 0.1\}$.
    \item For \eqref{alg:sgd-nesterov}, fix $\beta = 0.9$ and vary $\gamma \in \{0.001, 0.005, 0.025, 0.1\}$.
\end{itemize}

\noindent\textbf{Test on effect of momentum parameter $\beta$:}
\begin{itemize}
    \item For \eqref{alg:sgd-polyak}, fix $\eta = 0.01$ and vary $\beta \in \{0, 0.5, 0.9, 0.99\}$.
    \item For \eqref{alg:sgd-nesterov}, fix $\gamma = 0.01$ and vary $\beta \in \{0, 0.5, 0.9, 0.99\}$.
\end{itemize}
To ensure robustness and statistical significance, we conduct 100 independent repetitions for each configuration on each dataset. We report the average and standard deviation of the distance $d_t$ against the number of epochs (defined as $t/n$, with $n$ being the training set size).

\begin{table}
\centering\addtolength{\tabcolsep}{-0.2em}
\begin{tabular}{|c|c|c|c|c|c|c|c|c|c|}
\hline
Dataset & $a9a$ & $connect$-4 & $dna$ & $gisette$ & $mnist$ & $mushrooms$ & $phishing$ & $covtype$ \\ \hline
$n$ & 32561 & 67557 & 2000 & 6000 & 60000 & 8124 & 11055 & 581012 \\ \hline
$d$ & 123 & 126 & 180 & 5000 & 780 & 112 & 68 & 54 \\ \hline
\end{tabular}
\caption{Summary of datasets used in the experiments. Here, `$n$' stands for the sample size and `$d$' denotes the number of features.}
\label{tab:dataset}
\end{table}
Figure \ref{fig:test1} reports the evolution of $d_t$ for the logistic loss under varying step sizes, while Figure \ref{fig:test2} presents the corresponding results for different momentum parameters. For both \eqref{alg:sgd-polyak} and \eqref{alg:sgd-nesterov}, the results show that $d_t$ increases with the step size and with the momentum parameter $\beta$. These observations align with the stability bounds established in Theorems \ref{thm:stab-cvx-polyak} and \ref{thm:stab-cvx-nesterov}.

\begin{figure}
    \centering
    \includegraphics[width=1\linewidth]{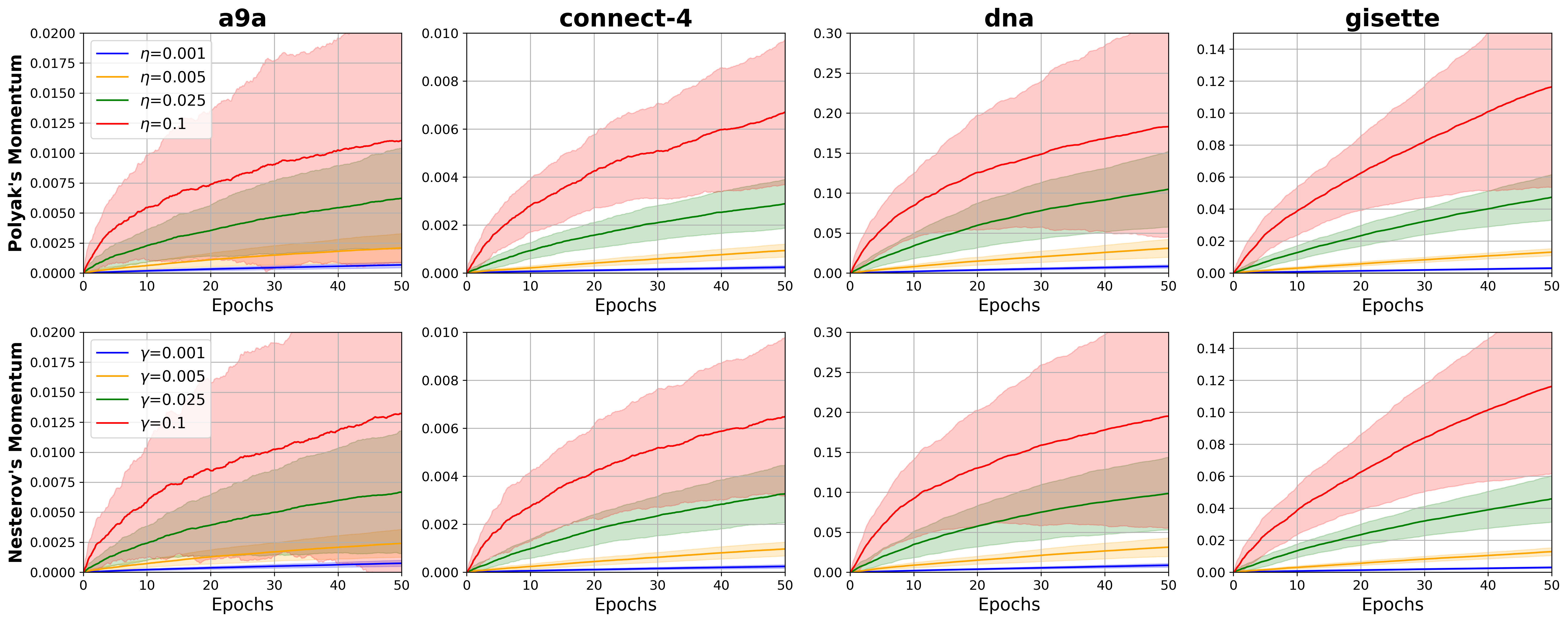}
    \caption{Tests on different step sizes for \eqref{alg:sgd-polyak} (top row) and \eqref{alg:sgd-nesterov} (bottom row) with fixed momentum parameter $\beta=0.9$. The y-axis records the Euclidean distance $\|\bw_t-\bw_t'\|_2$.}
    \label{fig:test1}
\end{figure}

\begin{figure}
    \centering
    \includegraphics[width=1\linewidth]{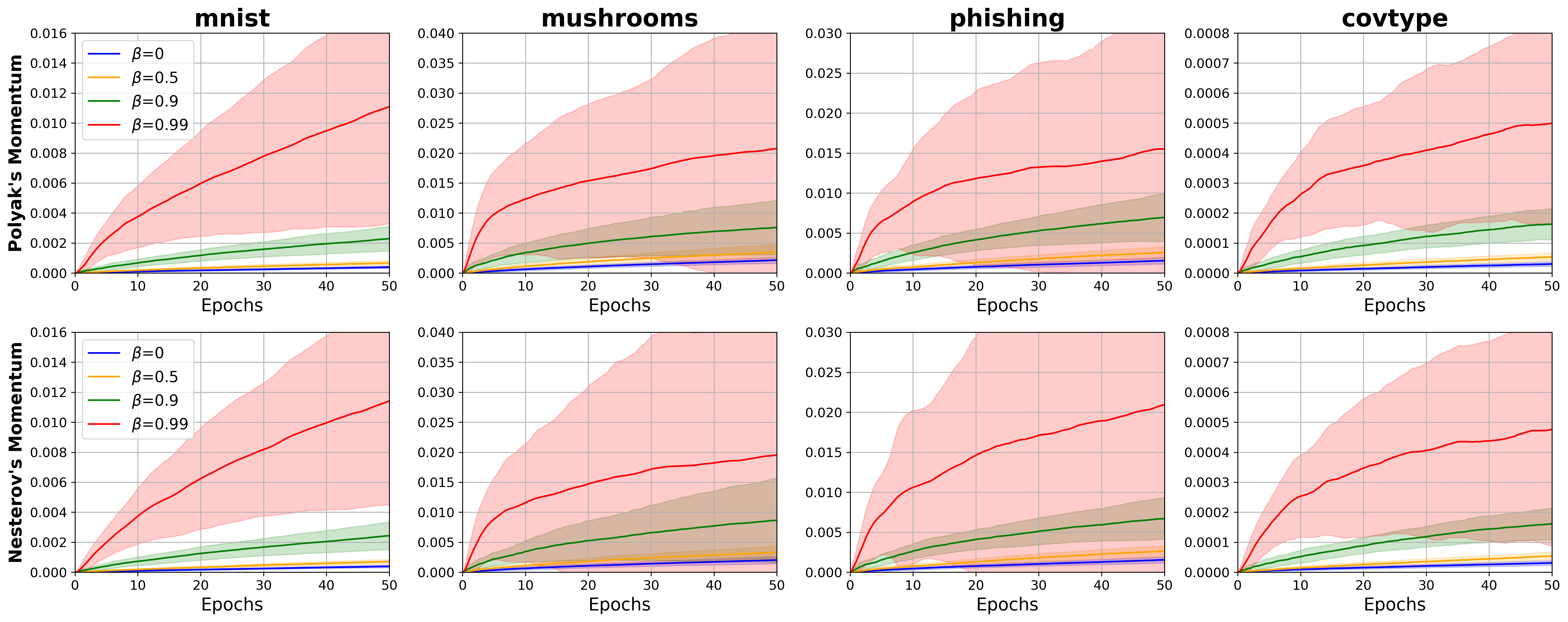}
    \caption{Tests on different momentum parameters for \eqref{alg:sgd-polyak} (top row) and \eqref{alg:sgd-nesterov} (bottom row) with fixed step sizes $\eta=\gamma=0.01$. The y-axis is the Euclidean distance $\|\bw_t-\bw_t'\|_2$.}
    \label{fig:test2}
\end{figure}

\section{Conclusion}\label{sec:conclusion}
In this paper, we presented a comprehensive stability and generalization analysis of stochastic gradient descent with momentum (SGDM). To this end, we introduced a generalized SGDM framework that unifies both SGD with Polyak’s momentum and SGD with Nesterov’s momentum, and established tight on-average model stability bounds for smooth and convex problems. These results yield stability bounds for both momentum variants as direct consequences. To the best of our knowledge, these represent the first tight stability bounds for both SGD with Polyak's momentum and SGD with Nesterov's momentum under comparable assumptions. Our analysis clarifies the role of the momentum parameter $\beta$: it may degrade stability by a factor of $O(1/(1-\beta)^{3/2})$. Importantly, we removed the commonly assumed Lipschitzness condition, and the obtained bounds apply to all $\beta\in[0,1)$ and are data-dependent in the sense that they benefit from small optimization errors.

By combining our stability bounds with optimization error analysis, we obtained optimal EPR bounds of order $O(1/\sqrt{n})$ for both momentum methods, where $n$ is the sample size. Additionally, we developed several new proof techniques, including the use of weighted gradient summations to control momentum terms and the construction of an auxiliary sequence that simplifies optimization analysis. These techniques may facilitate future analyses of momentum-based algorithms.

\section*{Acknowledgements}
The work of Y. Lei was supported by the Hong Kong Research Grants Council under the GRF projects 17302624 and 17305425. The work of Z. Wang was supported by the Wallenberg AI, Autonomous Systems, and Software Program (WASP), funded by the Knut and Alice Wallenberg Foundation. The work of X. Yuan was supported by the Hong Kong Research Grants Council under the GRF project 17309824.

\setlength{\bibsep}{0.111cm}
\bibliographystyle{abbrvnat}
\small
\bibliography{learning}
\end{document}